\documentclass[jmlr]{article}

\usepackage{jmlr2e}
\usepackage{hyperref}
\usepackage{natbib}
\usepackage{times}
\usepackage{amsfonts}
\usepackage{amsmath}
\usepackage{amstext}
\usepackage{latexsym}
\usepackage{color}
\usepackage{graphicx}
\usepackage{subfigure}
\usepackage{enumerate}
\usepackage{url}
\usepackage{algorithm}
\usepackage{algorithmic}
\usepackage{dsfont}

\hypersetup{
  colorlinks   = true, 
  urlcolor     = blue, 
  linkcolor    = blue, 
  citecolor   = blue 
}

\let\Pr\undefined
\def\Rset{\mathbb{R}}

\def\vcdim{\text{VCdim}}
\def\pdim{\text{Pdim}}
\DeclareMathOperator*{\E}{\mathbb{E}}
\DeclareMathOperator*{\Pr}{\mathbb{P}}

\DeclareMathOperator*{\argmin}{\rm argmin}
\DeclareMathOperator{\sgn}{sgn}

\newcommand{\set}[1]{\{#1\}}

\newcommand{\h}{\widehat}
\newcommand{\tl}{\widetilde}
\newcommand{\mat}[1]{{\mathbf #1}}
\renewcommand{\b}{\mat{b}}
\renewcommand{\u}{\mat{u}}

\newcommand{\w}{\mat{w}}
\newcommand{\x}{\mat{x}}
\newcommand{\cB}{\mathcal{B}}
\newcommand{\cL}{\mathcal{L}}
\newcommand{\cX}{\mathcal{X}}
\newcommand{\Ind}{\mathds{1}}
\newcommand{\1}{\mathds{1}}
\newcommand{\R}{\mathfrak{R}}
\newcommand{\e}{\epsilon}
\newcommand{\EQ}{\gets}
\newcommand{\wt}{\widetilde}
\newcommand{\ssigma}{{\boldsymbol \sigma}}

\newcommand{\bone}{b^{(1)}}
\newcommand{\btwo}{b^{(2)}}
\newcommand{\rev}{\mathrm{Revenue}}
\newcommand{\rstar}{r^*}

\newcommand{\ignore}[1]{}
\title{Learning Algorithms for Second-Price Auctions with Reserve}
\author{Mehryar Mohri \\
\addr{Courant Institute of Mathematical
    Sciences\\
251 Mercer Street \\
New York, NY}  \\
\AND Andres Mu\~noz Medina\\
\addr{Courant Institute of
Mathematical Sciences\\
251 Mercer Street\\
New York, NY
}
}

\editor{?}
\begin{document}
\maketitle

\begin{abstract}

  Second-price auctions with reserve play a critical role in the
  revenue of modern search engine and popular online sites since the
  revenue of these companies often directly depends on the outcome of
  such auctions. The choice of the reserve price is the main mechanism
  through which the auction revenue can be influenced in these
  electronic markets. We cast the problem of selecting the reserve
  price to optimize revenue as a learning problem and present a full
  theoretical analysis dealing with the complex properties of the
  corresponding loss function. We further give novel algorithms for
  solving this problem and report the results of several experiments
  in both synthetic and real data demonstrating their effectiveness.

\end{abstract}

\section{Introduction}

Over the past few years, advertisement has gradually moved away from
the traditional printed promotion to the more tailored and directed
online publicity. The advantages of online advertisement are clear:
since most modern search engine and popular online site companies
such as as Microsoft, Facebook, Google, eBay, or Amazon, may collect
information about the users' behavior, advertisers can better target
the population sector their brand is intended for.

More recently, a new method for selling advertisements has gained
momentum. Unlike the standard contracts between publishers and
advertisers where some amount of impressions is required to be
fulfilled by the publisher, an Ad Exchange works in a way similar to a
financial exchange where advertisers bid and compete between each
other for an ad slot. The winner then pays the publisher and his ad is
displayed. 

The design of such auctions and their properties are crucial since
they generate a large fraction of the revenue of popular online sites.
These questions have motivated extensive research on the topic of
auctioning in the last decade or so, particularly in the theoretical
computer science and economic theory communities. Much of this work
has focused on the analysis of mechanism design, either to prove some
useful property of an existing auctioning mechanism, to analyze its
computational efficiency, or to search for an optimal revenue
maximization truthful mechanism (see \citep{muthukrishnan2009} for a
good discussion of key research problems related to Ad Exchange and
references to a fast growing literature therein).

One important problem is that of determining an auction mechanism that
achieves optimal revenue \cite{muthukrishnan2009}. In the ideal
scenario where the valuation of the bidders is drawn i.i.d.\ from a
given distribution, this is known to be achievable (see for example
\citep{Myerson1981}). But, even good approximations of such
distributions are not known in practice. Game theoretical approaches
to the design of auctions have given a series of interesting results
including \citep{Riley1981,Milgrom1982,Myerson1981,nisan2007}, all of
them based on some assumptions about the distribution of the bidders,
e.g., the monotone hazard rate assumption.

The results of the recent publications have nevertheless set the basis
for most Ad Exchanges in practice: the mechanism widely adopted for
selling ad slots is that of a \emph{Vickrey auction} \citep{Vickrey2012}
or \emph{second-price auction with reserve price $r$}
\citep{EasleyKleinberg2010}. In such auctions, the winning bidder (if
any) pays the maximum of the second-place bid and the reserve price
$r$. The reserve price can be set by the publisher or automatically by
the exchange. The popularity of these auctions relies on the fact
that they are incentive compatible, i.e., bidders bid exactly what they
are willing to pay. It is clear that the revenue of the publisher
depends greatly on how the reserve price is set: if set too low, the
winner of the auction might end up paying only a small amount, even if
his bid was really high; on the other hand, if it is set too high,
then bidders may not bid higher than the reserve price and the ad slot
will not be sold.

We propose a machine learning approach to the problem of determining
the reserve price to optimize revenue in such auctions. The general
idea is to leverage the information gained from past auctions to
predict a beneficial reserve price. Since every transaction on an
Exchange is logged, it is natural to seek to exploit that data. This
could be used to estimate the probability distribution of the bidders,
which can then be used indirectly to come up with the optimal reserve
price \citep{Myerson1981, ostrovsky2011reserve}. Instead, we will seek
a discriminative method making use of the loss function related to the
problem and taking advantage of existing user features.

Machine learning methods have already been used for the related
problems of designing incentive compatible auction mechanisms
\cite{Balcan2008, blum2003}, for algorithmic bidding
\cite{wortman2007,AminKearnsKeySchwaighofer2012}, and even for
predicting bid landscapes \cite{cui2011bid}. Another closely related
problem for which machine learning solutions have been proposed is
that of revenue optimization for sponsored search ads and click
through rate predictions \cite{ZhuWang, HeChen, Kakade09}. But, to our
knowledge, no prior work has used historical data in combination with
user features for the sole purpose of revenue optimization in this
context. In fact, the only publications we are aware of that are
directly related to our objective are \citep{ostrovsky2011reserve} and
the interesting work of \citet{Cesa-BianchiGentileMansour2013} which
considers a more general case than \citep{ostrovsky2011reserve}. The
scenario studied by Cesa-Bianchi et al. is that of censored
information, which motivates their use of a regret minimization
algorithm to optimize the revenue of the seller. Our analysis assumes
instead access to full information. We argue that this is a more
realistic scenario since most companies do in fact have access to the
full historical data.

The learning scenario we consider is more general since it includes
the use of features, as is standard in supervised learning. Since user
information is sent to advertisers and bids are made based on this
information, it is only natural to include user features in our
learning solution.  A special case of our analysis coincides with the
no-feature scenario considered by
\citet{Cesa-BianchiGentileMansour2013}, assuming full information.
But, our results further extend those of this paper even in that
scenario. In particular, we present an $O(m\log m)$ algorithm for
solving a key optimization problem used as a subroutine by the
authors, for which they do not seem to give an algorithm. We also do
not require an i.i.d.\ assumption about the bidders, although this is
needed in \citep{Cesa-BianchiGentileMansour2013} in order to deal with
censored information only. 

The theoretical and algorithmic analysis of this learning problem
raises several non-trivial technical issues. This is because, unlike
some common problems in machine learning, here, the use of a convex
surrogate loss cannot be successful. Instead, we must derive an
alternative non-convex surrogate requiring novel theoretical guarantees
(Section~\ref{sec:guarantees}) and a new algorithmic solution
(Section~\ref{sec:algorithms}). We present a detailed analysis of
possible surrogate losses and select a continuous loss that we prove
to be calibrated and for which we give generalization bounds. This
leads to an optimization problem cast as a DC-programming problem
whose solutions are examined in detail: we first present an efficient
combinatorial algorithm for solving that optimization in the
no-feature case, next we combine that solution with the DC algorithm
(DCA) \cite{TaoAn1998} to solve the general
case. Section~\ref{sec:experiments} reports the results of our
experiments with synthetic data in both the no-feature case and the
general case as well as on data collected from eBay. We first
introduce the problem of selecting the reserve price to optimize
revenue and cast it as a learning problem (Section~\ref{sec:res}).

\section{Reserve price selection problem}
\label{sec:res}

As already discussed, the choice of the reserve price $r$ is the main
mechanism through which a seller can influence the auction revenue.
To specify the results of a second-price auction we need only the
vector of first and second highest bids which we denote by $\b =
(\bone, \btwo) \in \cB \subset \Rset_+^2$.  For a given reserve price
$r$ and bid pair $\b$, the revenue of an auction is given by
\begin{equation}
\label{eq:revenue} 
\rev(r, \b) = \btwo \Ind_{r < \btwo} + r \Ind_{\btwo \leq r \leq \bone}.
\end{equation}
The simplest setup is one where there are no features associated with
the auction. In that case, the objective is to select $r$ to optimize the
expected revenue, which can be expressed as follows:

\begin{align*}
\E_{\b}[\rev(r, \b)] 
& = \E_{\btwo}[\btwo \1_{r < \btwo}] + r\Pr[\btwo \leq r \leq
\bone]  \\ 
& = \int_{0}^{+\infty} \Pr[\btwo \1_{r < \btwo} > t] \, dt + r\Pr[\btwo
\leq r \leq \bone] \\
& = \int_{0}^{r} \Pr[r < \btwo] \, dt + \int_r^\infty \Pr[\btwo > t]
dt + r \Pr[\btwo \leq r \leq \bone] \\
& = \int_r^\infty \Pr[\btwo > t] \, dt  
+ r (\Pr[\btwo > r] + 1 - \Pr[\btwo > r] -\Pr[\bone < r]) \\
& = \int_r^\infty
\Pr[\btwo > t] \, dt + r\Pr[\bone \geq r].
\end{align*}

A similar derivation is given by
\citet{Cesa-BianchiGentileMansour2013}. In fact, this expression is
precisely the one optimized by these authors. If we now associate with
each auction a feature vector $\x \in \cX$, the so-called \emph{public
  information}, and set the reserve price to $h(\x)$, where $h\colon
\cX \to \Rset_+$ is our reserve price hypothesis function, the problem
can be formulated as that of selecting out of some hypothesis set $H$
a hypothesis $h$ with large expected revenue:
\begin{equation}
\E_{(\x, \b) \sim D}[\rev(h(\x), \b)],
\end{equation}
where $D$ is the unknown distribution according to which the pairs
$(\x, \b)$ are drawn. Instead of the revenue, we will consider a loss
function $L$ defined for all $(r, \b)$ by $L(r,\b) = -\rev(r,\b)$, and
will seek a hypothesis $h$ with small expected loss $\cL(h) := \E_{(\x,
  \b) \sim D} [L(h(\x), \b)]$. As in standard supervised learning
scenarios, we assume access to a training sample $S = ((\x_1,
  \b_1), \ldots, (\x_m, \b_m) )$ of size $m \geq 1$ 
drawn i.i.d.\ according to $D$ and denote by $\h \cL_S(h)$ the
empirical loss $\frac{1}{m}\sum_{i = 1}^m L(h(\x_i, \b_i)$.  In the
next sections, we present a detailed study of this learning problem.

\section{Learning guarantees}
\label{sec:guarantees}

\begin{figure}[t]
\centering
\begin{tabular}{c@{\hskip 1in}c}
\setlength{\tabcolsep}{100pt}
\includegraphics[scale=.45]{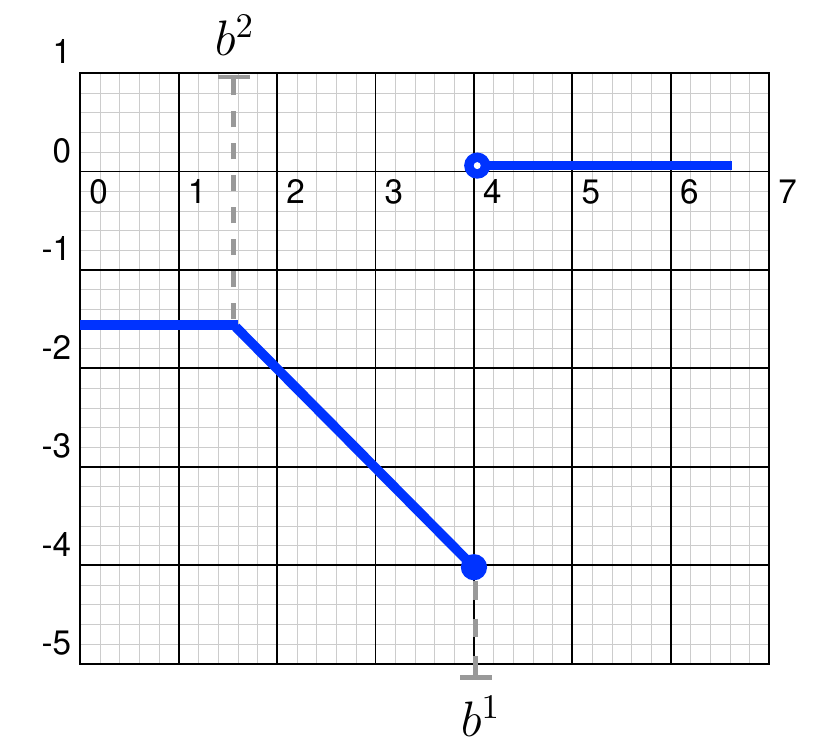}  &
\includegraphics[scale=.45]{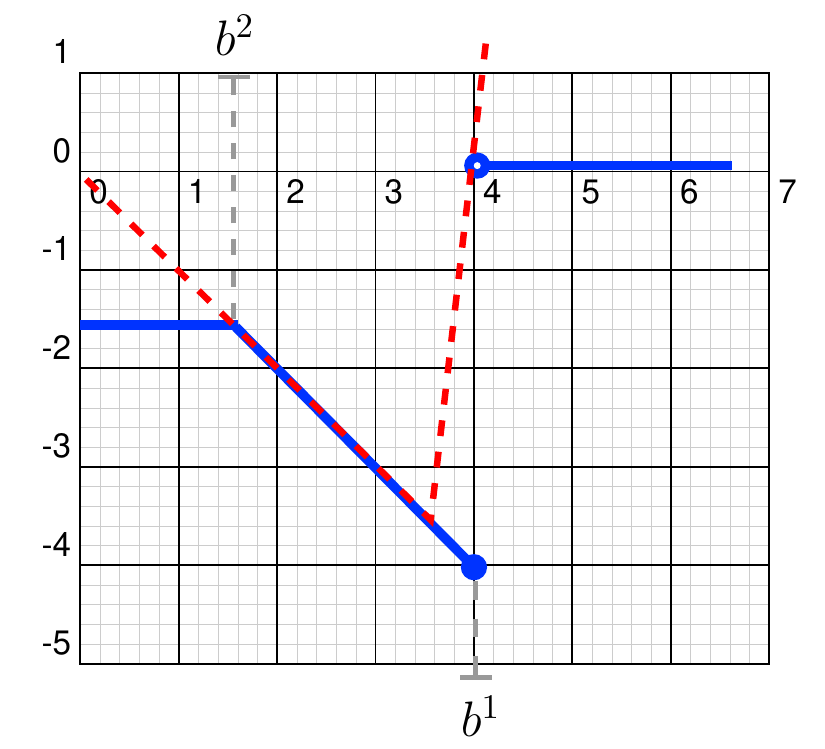}\\[-0.1cm]
(a) & (b)\\[-.25cm]
\end{tabular}
\caption{(a) Plot of the loss function $r \mapsto L(r, \b)$ for fixed
  values of $\bone$ and $\btwo$; (b) piecewise linear convex surrogate loss.}
\vspace{-.25cm}
\label{fig:realLoss}
\end{figure}

To derive generalization bounds for the learning problem formulated in
the previous section,
we need to analyze the complexity of the family of functions $L_H$
mapping $\cX \times \cB$ to $\Rset$ defined by $L_H = \set{(\x, \b)
  \mapsto L(h(\x), \b) \colon h \in H}$. The loss function $L$ is
neither Lipschitz continuous nor convex (see
Figure~\ref{fig:realLoss}). To analyze its complexity, we decompose
$L$ as a sum of two loss functions $l_1$ and $l_2$ with more
convenient properties. We have $L = l_1 + l_2$ with $l_1$ and $l_2$
defined for all $(r, \b)\in \Rset \times \cB$ by 
\begin{align*} 
l_1(r, \b) &= -\btwo \Ind_{r < \btwo} -r \Ind_{\btwo \leq r \leq \bone}
-\bone \Ind_{r > \bone} \\
l_2(r,\b) &= \bone\Ind_{r > \bone}.
\end{align*}
Note that for a fixed $\b$, the function $r \mapsto l_1(r, \b)$ is
$1$-Lipschitz since the slope of the lines defining the function is at
most $1$.  We will consider the corresponding family of loss
functions: $l_{1 H} = \set{(\x, \b) \mapsto l_1(h(\x), \b) \colon h
\in H}$ and $l_{2 H} = \set{(\x, \b) \mapsto l_2(h(\x), \b) \colon h
\in H}$ and use the notions of pseudo-dimension as well as empirical
and average Rademacher complexity. The pseudo-dimension is a standard
complexity measure \citep{Pollard84} extending the notion
of VC-dimension to real-valued functions (see also \citep{Mohribook}).
For a family of functions $G$ and finite sample $S = (z_1, \ldots,
z_m)$ of size $m$, the empirical Rademacher complexity is defined by
$\h \R_S(G) = \E_\ssigma \left[\sup_{g \in G} \frac{1}{m} \sum_{i=1}^m
\sigma_i g(z_i) \right]$, where $\ssigma = (\sigma_1, \ldots,
\sigma_m)^\top$, with $\sigma_i$s independent uniform random variables
taking values in $\set{-1, +1}$. The Rademacher complexity of $G$ is
defined as $\R_m(G) =\E_{S \sim D^m}[ \h \R_S(G)]$.

In order to bound the complexity of $L_H$ we will first bound the
complexity of the family of loss functions $l_{1H}$ and $l_{2H}$.

\begin{proposition}
\label{prop:l1}
  For any hypothesis set $H$ and any sample $S = ((\x_1, \b_1), \ldots,
  (\x_m, \b_m))$, the empirical
  Rademacher complexity of $l_{1 H}$ can be bounded as follows:
\begin{equation*}
\h \R_S(l_{1H}) \leq \h \R_S(H).
\end{equation*}

\end{proposition}
\begin{proof} 
  By definition of the empirical Rademacher complexity, we can write
\begin{equation*}
\h \R_S(l_{1H}) 
= \frac{1}{m} \E_{\ssigma} \left[\sup_{h \in H} \sum_{i =
    1}^m \sigma_i l_1(h(\x_i), \b_i) \right]
= \frac{1}{m}
  \E_{\ssigma} \left[\sup_{h \in H} \sum_{i = 1}^m \sigma_i (\psi_i \circ
  h)(\x_i) \right],
\end{equation*}
where, for all $i \in [1, m]$, $\psi_i$ is the function defined by
$\psi_i\colon r \mapsto l_1(r, \b_i)$.  For any $i \in [1, m]$,
$\psi_i$ is $1$-Lipschitz, thus, by the contraction
lemma~\ref{lemma:contraction}, we have the inequality $\h \R_S(l_{1H})
\leq \frac{1}{m} \E_{\ssigma}[\sup_{h \in H} \sum_{i = 1}^m \sigma_i
h(\x_i) ] = \h \R_S(H)$.
\end{proof}

\begin{proposition}
\label{prop:l2}
  Let $M = \sup_{\b \in \cB} \bone$. Then, for any hypothesis set $H$
  with pseudo-dimension $d = \pdim(H)$ and any sample $S = ((\x_1,
  \b_1), \ldots, (\x_m, \b_m))$, the empirical Rademacher complexity
  of $l_{2 H}$ can be bounded as follows:
\begin{equation*}
\h \R_S(l_{2H}) \leq \sqrt{\frac{2 d \log \frac{em}{d}}{m}} .
\end{equation*}
\end{proposition}

\begin{proof}
By definition of the empirical Rademacher complexity, we can write
\begin{equation*}
\h \R_S(l_{2H}) 
= \frac{1}{m} \E_{\ssigma} \Big[\sup_{h \in H} \sum_{i =
    1}^m \sigma_i \bone_i \Ind_{h(\x_i) > \bone_i} \Big] 
= \frac{1}{m}
  \E_{\ssigma} \Big[\sup_{h \in H} \sum_{i = 1}^m \sigma_i \Psi_i (\Ind_{h(\x_i) > \bone_i}) \Big],
\end{equation*}
where for all $i \in [1, m]$, $\Psi_i$ is the $M$-Lipschitz function
$x \mapsto \bone_i x$. Thus, by Lemma~\ref{lemma:contraction} combined
with Massart's lemma (see for example \cite{Mohribook}), we can write
\begin{equation*}
\h \R_S(l_{2H}) 
\leq \frac{M}{m}
  \E_{\ssigma} \Big[\sup_{h \in H} \sum_{i = 1}^m \sigma_i
  \Ind_{h(\x_i) > \bone_i} \Big] 
\leq M \sqrt{\frac{2 d' \log \frac{em}{d'}}{m}},
\end{equation*}
where $d' = \vcdim(\set{(\x, \b) \mapsto \Ind_{h(\x) - \bone > 0}
  \colon (\x, \b) \in \cX \times \cB})$.  Since the second bid
component $\btwo$ plays no role in this definition, $d'$ coincides
with $\vcdim(\set{(\x, \bone) \mapsto \Ind_{h(\x) - \bone > 0} \colon
  (\x, \bone) \in \cX \times \cB_1})$, where $\cB_1$ is the projection
of $\cB \subseteq \Rset^2$ onto its first component, and is
upper-bounded by $\vcdim(\set{(\x, t) \mapsto \Ind_{h(\x) - t > 0}
  \colon (\x, t) \in \cX \times \Rset})$, that is, the pseudo-dimension
of $H$.
\end{proof}

\begin{theorem}
\label{th:learningbound} 
Let $M = \sup_{\b \in \cB} \bone$ and let $H$ be a hypothesis set with
pseudo-dimension $d = \pdim(H)$. Then, for any $\delta > 0$, with
probability at least $1 - \delta$ over the choice of a sample $S$ of
size $m$, the following inequality holds for all $h \in H$:
\begin{equation*}
\cL(h) \leq \h \cL_S(h) + 2\R_m(H) +
2M \sqrt{\frac{2 d \log \frac{em}{d}}{m}} + M \sqrt{\frac{\log \frac{1}{\delta}}{2m}}.
\end{equation*}
\end{theorem}

\begin{proof} 
  By a standard property of the Rademacher complexity, since $L = l_1
  + l_2$, the following inequality holds: $\R_m(L_H) \leq \R_m(l_{1H})
  + \R_m(l_{2H})$. Thus, in view of Propositions~\ref{prop:l1} and
  \ref{prop:l2}, the Rademacher complexity of $L_H$ can be bounded via
\begin{equation*}
\R_m(L_H) \leq \R_m(H) + M \sqrt{\frac{2 d \log \frac{em}{d}}{m}}.
\end{equation*}
The result then follows by the application of a standard Rademacher
complexity bound \citep{KoltchinskiiPanchenko2002}.
\end{proof}
This learning bound invites us to consider an algorithm seeking $h
\!\in\! H$ to minimize the empirical loss $\h \cL_S(h)$, while
controlling the complexity (Rademacher complexity and
pseudo-dimension) of the hypothesis set $H$.  However, as in the
familiar case of binary classification, in general, minimizing this
empirical loss is a computationally hard problem. Thus, in the next
section, we study the question of using a surrogate loss instead of
the original loss $L$.

\begin{figure}[t]
\centering
\begin{tabular}{c@{\hskip 1in}c}
\includegraphics[scale=.35]{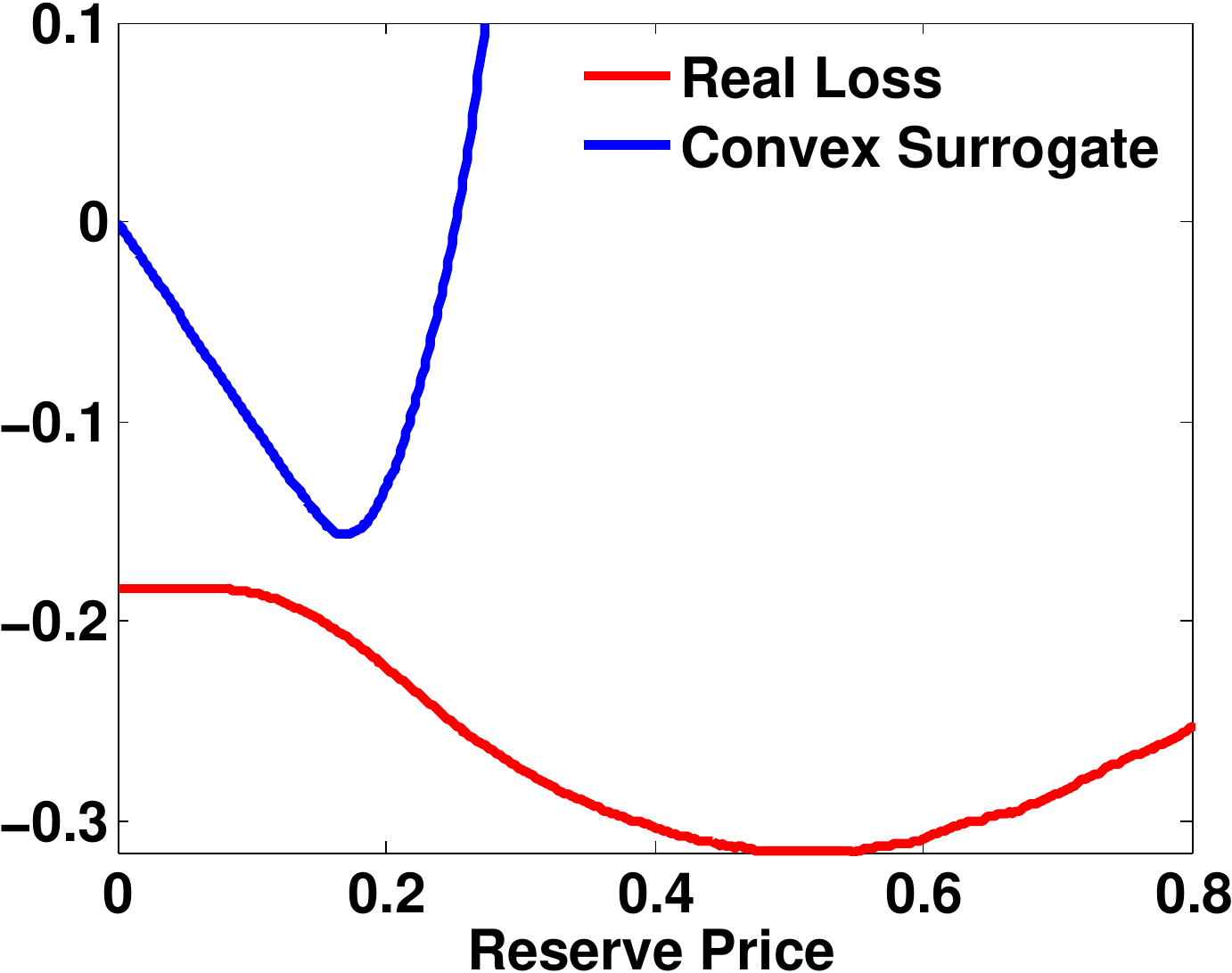}&
\includegraphics[scale=.35]{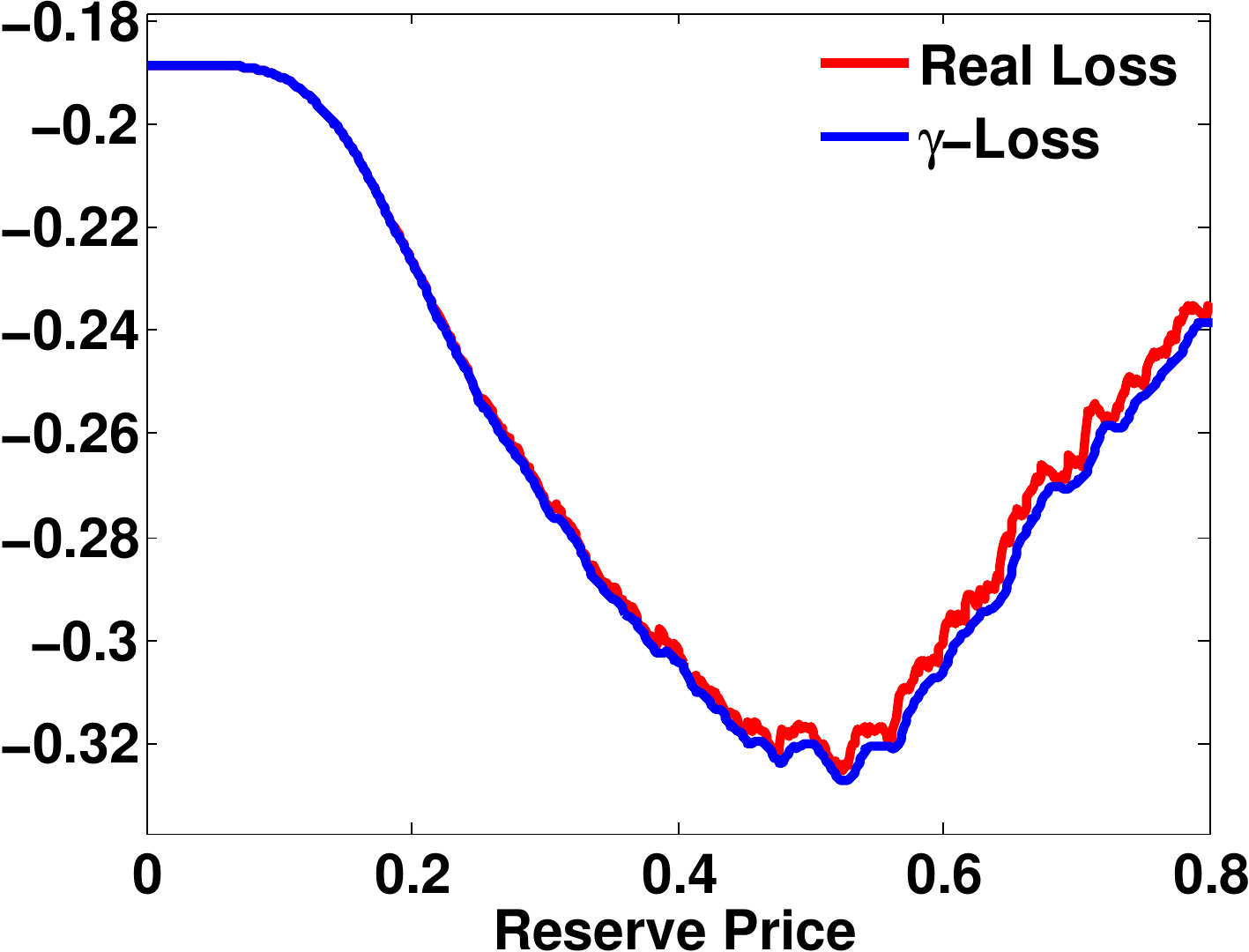}\\
(a) & (b)
\end{tabular}
\vspace{-.3cm}
\caption{Comparison of the sum of real losses $\sum_{i = 1}^m L(\cdot,
  \b_i)$ for $m = 500$ versus two different surrogates. (a) Sum of
  convex surrogate losses: the minimizer significantly differs from
  that of the sum of the original losses.  (b) The surrogate loss sum
  $\sum_{i = 1}^m L_\gamma(\cdot, \b_i)$ for $\gamma = .02$}
\label{fig:sumloss}
\vspace{-.5cm}
\end{figure}

\subsection{Surrogate loss}
\label{sec:surrogate}

As pointed out earlier, the loss function $L$ does not admit some
common useful properties: for any fixed $\b$, $L(\cdot, \b)$ is not
differentiable at two points, is not convex, and is not Lipschitz, in
fact it is discontinuous.  For any fixed $\b$, $L(\cdot, \b)$ is
quasi-convex, a property that is often desirable since there exist
several solutions for quasi-convex optimization problems. However, in
general, a sum of quasi-convex functions, such as the sum $\sum_{i =
1}^m L(\cdot, \b_i)$ appearing in the definition of the empirical
loss, is not quasi-convex and a fortiori not convex.\footnote{It is
known that under some separability condition if a finite sum of
quasi-convex functions on an open convex set is quasi-convex then all
but perhaps one of them is convex \cite{DebrueKoopmans1982}.} In fact,
in general, such a sum may admit exponentially many local minima. This
leads us to seek a surrogate loss function with more favorable
optimization properties.

A standard method in machine learning consists of replacing the loss
function $L$ with a convex upper bound \cite{Bartlett}. A natural
candidate in our case is the piecewise linear convex function shown in
Figure~\ref{fig:realLoss}(b).  However, while this convex loss
function is convenient for optimization, it is not calibrated and does
not provide a useful surrogate.  The calibration problem is
illustrated by Figure~\ref{fig:sumloss}(a) in dimension one, where the
true objective function to be minimized $\sum_{i=1}^m L(r,\b_i)$ is
compared with the sum of the surrogate losses. The next theorem shows
that this problem affects in fact any non-constant convex surrogate.
It is expressed in terms of the loss $\wt{L}\colon \Rset \times
\Rset_+ \to \Rset$ defined by $\wt{L}(r, b) = -r \Ind_{r \leq
  b}$, which coincides with $L$ when the second bid is $0$.

\begin{definition}
\label{def:consistency}
Let $M >0$, we say a function $L \colon [0,M] \times [0,M] \to \Rset$
is consistent with $\tl{L}$ if for every distribution $D$ there exists
a minimizer $\rstar \in \argmin_r \E_{b \sim D}
[L_c(r, b)]$ such that $\rstar \in \argmin_r \E_{b \sim D}[\tl{L}(r, b)]$.
\end{definition}

\begin{definition}
\label{def:weakconsistency}
We say that a sequence of functions $L_n \colon [0,M] \times [0,M]$ is
weakly consistent with $\tl{L}$ if there exists $r_n \in \argmin_r \E_{b \sim
  D}[L_n(r,b)]$ such that $r_n \rightarrow r^*$ and $r^* \in \argmin \E_{b
  \sim D}[\tl{L}(r,b)]$.
\end{definition}

\begin{proposition}[convex surrogates]
\label{prop:cvxconst}
Let $L_c \colon [0,M] \times [0,M] \to \Rset$ be a bounded
function, convex with respect to its first argument. If $L_c$ is
consistent with $\tl{L}$, then $L_c(\cdot, b)$ is constant for every
$b \in [0,M]$.
\end{proposition}
\begin{proof}
Let $0 < b_1 < b_2 < M$, for every $\mu \in [0,1]$ define $D_\mu$ to be the
probability distribution supported on $\{b_1, b_2\}$ with $D_{\mu}(b_1) =
\mu$. Denote by $\E_\mu$ the expectation with respect to this
distribution. A straightforward calculation shows that the unique minimizer of
$E_\mu(\tl{L}(r, b)$ is given by $b_1$ if $\mu > \frac{b_2- b_1}{b_2}$
and by $b_2$ otherwise. Therefore, if $F_\mu(r) = \E_\mu(L_c(r, b))$,
it must be the case that $b_1$ is a minimizer of $F_\mu$ for $\mu <
\frac{b_2 - b_1}{b_2}$ and $b_2$ is a minimizer of $F_\mu$ for $\mu >
\frac{b_2 - b_1}{b_2}$.
Let $D^+_r$ and $D^-_r$ denote the left and right derivatives with
respect to $r$. Since $F_\mu$ is a convex function, its right and left
derivatives are well defined and we must have
\begin{align}
  &0 \geq D^-_r F_\mu(b_2) = \mu D^-_r L_c(b_2, b_1) + (1
  - \mu) D^-_r L_c(b_2, b_2)  & \text{for } \mu > \frac{b_2 -
    b_1}{b_2}, \label{eq:leftderivative} \\
  &0 \leq D^+_r F_\mu(b_1) \leq D^-_r F_\mu(b_2) &
  \text{for } \mu < \frac{b_2 - b_1}{b_2}. \label{eq:rightderivative}
\end{align}
Where the second inequality in \eqref{eq:rightderivative} holds by
convexity of $F_\mu$ and the fact that $b_2 > b_1$. By setting $\mu =
\frac{b_2 - b_1}{b_2}$, it follows from inequalities
\eqref{eq:leftderivative} and \eqref{eq:rightderivative} that
$D^-_rF_\mu(b_2) = 0$. Rearranging terms and replacing the value
of $\mu$ we arrive to the equivalent condition
\begin{equation*}
  (b_2 - b_1)D^-_rL_c(b_2, b_1) = -b_1 D^-_rL_c(b_2, b_2).
\end{equation*}
Since $L_c$ is a bounded function it follows that
$D^-_rL_c(b_2, b_1)$ is bounded for every $b_1, b_2 \in
(0.M)$, therefore as $b_1 \rightarrow b_2$ we must have $b_2
D^-_r L_c(b_2, b_2) = 0$ which implies $D^-_r
L_c(b_2, b_2) = 0$ for all $b_2 > 0$. In view of this, inequality
\eqref{eq:leftderivative} can only be satisfied if 
$D^-_rL_c(b_2, b_1) \leq 0$. However, the convexity of $L_c$
implies $D^-_r L_c(b_2, b_1) \geq D^-_r(b_1, b_1) =
0$. Therefore, $D^-_r L_c(b_2, b_1) = 0$ must hold for all
$b_2 > b_1 > 0$. Similarly, by definition of $F_\mu$, the first inequality in
\eqref{eq:rightderivative} implies
\begin{equation}
\label{eq:rightderbound}
\mu D^+_rL_c(b_1, b_1) +
D^+_r L_c(b_1, b_2) \geq 0.
\end{equation}
 Nevertheless, for every $b_2 > b_1$ we have $0 = D^-_r L_c(b_1,b_1)
\leq D^+_r L_c(b_1,b_1) \leq D^-_r L_c(b_2, b_1) = 0$. Consequently,
$D^+_rL_c(b_1, b_1) = 0$ for all $b_1 > 0$. Furthermore,
$D^+_rL_c(b_1, b_2) \leq D^+_r L_c(b_2, b_2) = 0$; thus, in order for
\eqref{eq:rightderbound} to be satisfied we must have $D^+_r L_c(b_1,
b_2) = 0$ for all $b_1 < b_2$.

Thus far, we have shown that for every $b > 0$, if $r\geq b$, then
$D^-_rL_c(r, b) = 0$, whereas $D^+_rL_c(r,b) = 0$ for $r \leq b$. A
simple convexity argument shows that $L_c(\cdot, b)$ is in fact
differentiable and $\frac{d}{dr} L_c(r, b) = 0$ for all $r \in
(0,M)$. Therefore it must be that $L_c(\cdot, b)$ is a constant
function.
\end{proof}
The result of the previous Proposition can be considerably strengthen as
the following Theorem shows.

\begin{theorem} Let $M > 0$ and
  let $L_n :[0, M] \times [0,M] \to \Rset$ be a sequence of
  functions convex and differentiable with respect to its first
  coordinate satisfying
  \begin{itemize}
  \item $\sup_{b \in [0,M], n \in \mathbb{N}} \max(|\frac{d}{d_r}L_n(0,b)|,
    |\frac{d}{dr}L_n(M, b))| = K < \infty$
  \item $L_n$ is weakly consistent with $\tl{L}$.
  \item $L_n(0, b) = 0$ for all $b$.
  \end{itemize}
If the sequence $L_n$ converges pointwise to a function $L_c$ then
$L_n(\cdot , b)$ converges uniformly to $L_c(\cdot, b) \equiv 0$.
\end{theorem}
\begin{proof}
We first show that the functions $L_n$ are uniformly bounded for
every $b$. 
\begin{align*}
 |L_n(r, b)| &= \Big| \int_0^r \frac{d}{dr}L_n(r, b)dr \Big|  \leq \int_0^M
 \max\left(\Big| \frac{d}{dr}L_n(0,b) \Big| , \Big| \frac{d}{dr}L_n(M,
   b)\Big| \right) dr \\
& \leq \int_0^M K dr = MK.
\end{align*}
Where the first inequality holds since, by convexity, the derivative
of $L_n$ with respect to $r$ is an increasing function. Let us show
that the sequence $L_n$ is also equicontinuous. It will follow then by
the theorem of Arzela-Ascoli that $L_n(\cdot, b)$ converges uniformly
to $L_c(\cdot, b)$. Let $r_1, r_2 \in [0,M]$, for every $b \in [0,M]$
we have
\begin{align*}
  |L_n(r_1, b) - L_n(r_2, b)|  &\leq \sup_{r \in
    [0,M]}\left|\frac{d}{dr}L_n(r, b)\right| |r_1 - r_2|  \\
  & = \max \left(\left|\frac{d}{d_r}L_n(0,b)\right|,\left|\frac{d}{dr}L_n(M, b))\right| \right)| r_1 - r_2|
  \\
  & \leq K | r_1 - r_2|. 
\end{align*}
Where again we have used the convexity of $L_n$ to derive the first
equality. Let $F_n(r) = \E_{b \sim D} [L_n(r, b)]$ and $F(r) = \E_{b
  \sim D} [L_c(r, b)]$. It is immediate that $F_n$ is a convex function
and by invoking Arzela-Ascoli's theorem again we can show that the
sequence $F_n$ has a subsequence which is uniformly
convergent. Furthermore by the dominated convergence theorem we have
$F_n(r)$ converges pointwise to $F(r)$. Therefore, the uniform limit
of $F_n$ must be $F$. This implies that 
\begin{equation*}
  \min_{r \in [0,M]} F(r) = \lim_{n \rightarrow \infty} \min_{r \in
    [0,M]} F_n(r) = \lim_{n \rightarrow \infty} F_n(r_n) = F(r^*),
\end{equation*}
which is precisely the definition of consistency with
$\tl{L}$. Furthermore, the function $L_c(\cdot, b)$ is convex since it
is the uniform limit of convex functions. It then follows by
Theorem~\ref{prop:cvxconst} it follows that $L_c(\cdot, b) \equiv
L_c(0,b) = 0$.
\end{proof}
The above theorems show that even a weakly consistent sequence of
convex losses is uniformly close to a constant function and is therefore
uninformative for our task. This leads us to consider alternative
non-convex loss functions. Perhaps, the most natural surrogate loss is
then $L'_\gamma$, an upper bound on $L$ defined for all $\gamma > 0$
by:
\begin{multline*}
L'_\gamma(r, \b) = -\btwo \Ind_{r \leq \btwo} -r \Ind_{\btwo < r \leq
  \big((1 - \gamma)\bone\big) \vee \btwo} \\
+ \Big(\frac{1 - \gamma}{\gamma} \vee \frac{\btwo}{\bone - \btwo}\Big)(r - \bone )
\Ind_{\big((1 - \gamma) \bone\big) \vee \btwo
 < r \leq \bone} ,
\end{multline*}
where $c \vee d = \max(c, d)$. The plot of this function is shown in
Figure~\ref{fig:rho_loss}(a). The $\max$ terms ensure that the
function is well defined if $(1 - \gamma) \bone < \btwo$. However, this
turns out to be also a poor choice because $L'_\gamma$ is a loose
upper bound of $L$ in the most critical region, that is around the
minimum of the loss $L$. \ignore{ This furthermore results in a loss that in
general is not calibrated, that is, for a hypothesis set $H$, we may
not have in general $\inf_{h \in H} \E_{(\x, \b) \sim
D}[L'_\gamma(h(\x), \b)] \to \inf_{h \in H} \E_{(\x, \b) \sim
D}[L(h(\x), \b)]$ as $\gamma \to 0$, depending on the distribution
$D$.\footnote{Technically, this is related to the fact that the
difference $\E[L'_\gamma(h(\x), \b) - L(h(\x), \b)]$ can be expressed
and bounded in terms of $\Pr\big[h(\x) \in ((1 - \gamma) \bone,
\bone]\big]$ and to the closeness of the interval $((1 - \gamma)
\bone, \bone]$ at $\bone$.} } Thus, instead, we will consider, for any
$\gamma > 0$, the loss function $L_\gamma$ defined as follows:
\begin{equation}
L_\gamma(r, \b) = -\btwo \Ind_{r \leq \btwo} -r \Ind_{\btwo < r \leq
  \bone} + 
\frac{1}{\gamma}(r - (1 + \gamma) \bone ) \Ind_{\bone < r \leq
 (1 + \gamma)\bone},
\end{equation}
and shown in Figure~\ref{fig:rho_loss}(b).\footnote{Technically, the
  theoretical and algorithmic results we present for $L_\gamma$ could
be developed in a somewhat similar way for $L'_\gamma$.} A comparison
between the sum of $L$-losses and the sum of $L_\gamma$-losses is
shown in Figure~\ref{fig:sumloss}(b). Observe that the fit is
considerably better than when using a piecewise linear convex
surrogate loss.  A possible concern associated with the loss function
$L_\gamma$ is that it is a lower bound for $L$. One might think then
that minimizing it would not lead to an informative solution. However,
we argue that this problem arises significantly with upper bounding
losses such as the convex surrogate, which we showed not to lead to a
useful minimizer, or $L'_\gamma$, which is a poor approximation of $L$
near its minimum.  By matching the original loss $L$ in the region of
interest, around the minimal value, the loss function $L_\gamma$ leads
to more informative solutions in this problem. We further analyze the
difference of expectations of $L$ and $L_\gamma$ and show that
$L_\gamma$ is calibrated. Since $L_\gamma \rightarrow L$ as $\gamma
\rightarrow 0$, this result may seem trivial. However this convergence
is not uniform and therefore calibration is not guaranteed.  We will
use for any $h \in H$, the notation $\cL_\gamma(h) := \E_{(\x, \b)
\sim D}[L_\gamma(h(\x), \b)]$.

\begin{figure}[t]
\centering
\begin{tabular}{c@{\hspace{.5cm}}c} 
\includegraphics[scale=.48]{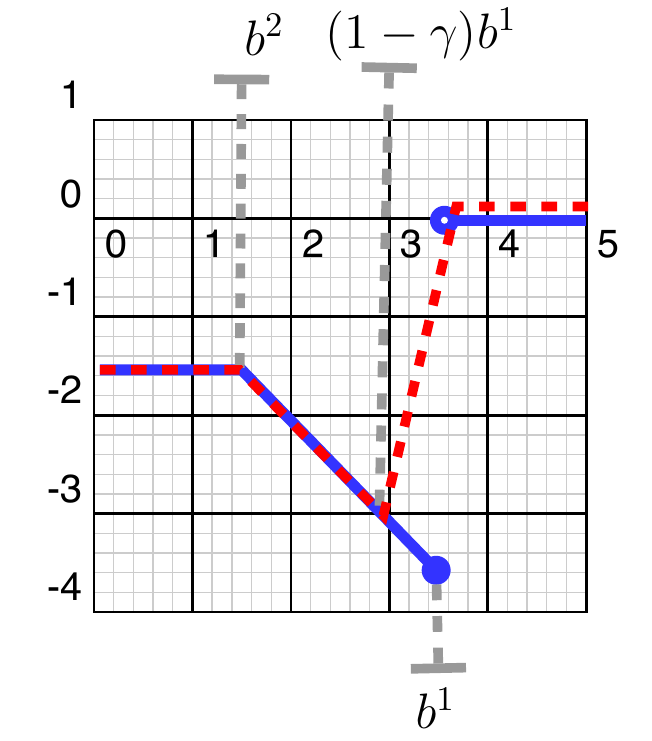} &
\includegraphics[scale=.48]{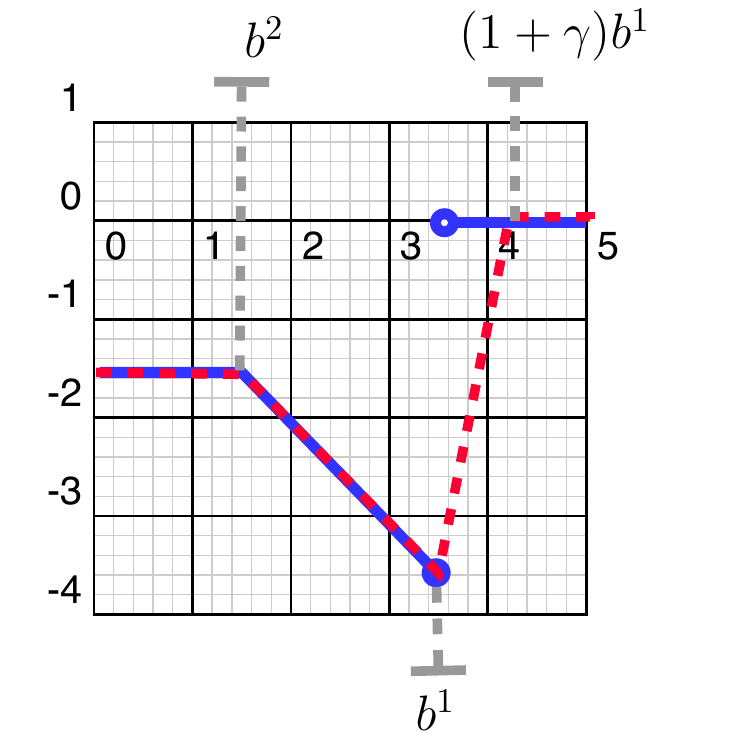}\\
(a) & (b)
\end{tabular}
\caption{Comparison of the true loss $L$ with (a) the surrogate loss
  $L'_\gamma$; (b) the surrogate loss $L_\gamma$, for $\gamma = 0.1$.}
\vspace{-.5cm}
\label{fig:rho_loss}
\end{figure} 
 
\begin{theorem}
\label{th:calibration}
 Let $H$ be a closed, convex subset of a linear space of functions containing
$0$. Denote by $h^*_\gamma$ the solution of $\min_{h \in
H}\cL_\gamma(h)$. If $\sup_{\b \in \mathcal B} \bone = M < \infty$,
then
\begin{equation*}
\cL(h^*_\gamma) - \cL_\gamma(h^*_\gamma) \leq \gamma M .
\end{equation*}
\end{theorem}
The following sets, which will be used in our proof, form a partition
of $\mathcal X \times \mathcal B$
\begin{align*}
 I_1 &= \{(\x, \b) | h^*_\gamma(\x) \leq \btwo\}  &
 I_2 &= \{(\x, \b) | h^*_\gamma(\x) \in (\btwo, \bone] \} \\
 I_3 &= \{(\x, \b) | h^*_\gamma(\x) \in (\bone, (1 + \gamma) \bone] \} &
 I_4 &= \{(\x, \b) | h^*_\gamma(\x) > (1 + \gamma) \bone \}
\end{align*}
This sets represent the different regions where $L_\gamma$ is
defined. In each region the function is affine.  We will now prove a
technical lemma that will help us in the proof of
Theorem~\ref{th:calibration}. 
\begin{lemma}
\label{lemma:wstarprop}
Under the conditions of Theorem~\ref{th:calibration},
\begin{equation*}
  \E_{\x,\b} \Big[h^*_\gamma(\x) \Ind_{I_2}(\x)\Big] \geq 
\frac{1}{\gamma} \E_{\x, \b} \Big[ h^*_\gamma(\x) \Ind_{I_3}(\x)\Big].
\end{equation*}
\end{lemma}
\begin{proof}
Let $0 < \lambda < 1$. Since $H$ is a convex set, it follows that
$\lambda h^*_\gamma \in H$; and from the definition of $h^*_\gamma$ we
must have:
\begin{equation}
  \label{eq:minprop}
  \E_{\x,\b} \Big[L_\gamma (h^*_\gamma(\x), \b)\Big] \leq 
\E_{\x, \b} \Big[L_\gamma(\lambda h^*_\gamma(\x), \b)\Big].
\end{equation}
If $h^*_\gamma(\x) < 0$, then $L_\gamma(h^*_\gamma(\x), \b) =
L_\gamma(\lambda h^*_\gamma(\x)) = -\btwo$ by definition of $L_\gamma$. If on the
other hand $h^*_\gamma(\x) > 0$, since $\lambda h^*_\gamma(\x) <
h^*_\gamma(\x)$ we must have that for $(\x, \b) \in I_1\;
L_\gamma(h^*_\gamma(\x), \b) = L_\gamma(\lambda h^*_\gamma(\x), \b) =
- \btwo$ too. Moreover, from the fact that $L_\gamma \leq 0$ and
$L_\gamma(h^*_\gamma(\x), \b) = 0 $ for $(\x, \b) \in I_4$ it follows
that $L_\gamma(h^*_\gamma(\x), \b) \geq L_\gamma(\lambda
h^*_\gamma(\x), \b)$ for $(\x, \b) \in I_4$, and therefore the
following inequality trivially holds:
\begin{equation}
  \label{eq:i1i4}
  \E_{\x,\b} \Big[L_\gamma(h^*_\gamma(\x), \b) (\Ind_{I_1}(\x) + \Ind_{I_4}(\x))\Big] 
\geq \E_{\x,\b} \Big[L_\gamma(\lambda h^*_\gamma(\x), \b) (\Ind_{I_1}(\x) + \Ind_{I_4}(\x))\Big].
\end{equation}
Subtracting \eqref{eq:i1i4} from \eqref{eq:minprop} we obtain
\begin{equation*}
    \E_{\x,\b} \Big[L_\gamma(h^*_\gamma(\x), \b) (\Ind_{I_2}(\x) + \Ind_{I_3}(\x))\Big] 
\leq  \E_{\x,\b} \Big[L_\gamma(\lambda h^*_\gamma(\x), \b) (\Ind_{I_2}(\x) + \Ind_{I_3}(\x))\Big].
\end{equation*}
By rearranging terms we can see this inequality is equivalent to
\begin{equation}
  \label{eq:i2i3}
  \E_{\x,\b} \Big[ (L_\gamma(\lambda h^*_\gamma(\x), \b) - L_\gamma( h^*_\gamma(\x), \b)) \Ind_{I_2}(\x)\Big]
  \geq   \E_{\x,\b} \Big[ (L_\gamma(h^*_\gamma(\x), \b) -
  L_\gamma(\lambda h^*_\gamma(\x), \b)) \Ind_{I_3}(\x)\Big]
\end{equation}
Notice that if $(\x, \b) \in I_2$, then $L_\gamma(h^*_\gamma(\x), \b)
= -h^*_\gamma(\x)$. If $\lambda h^*_\gamma(\x) > \btwo$ too then
$L_\gamma(\lambda h^*_\gamma(\x), \b) = - \lambda h^*_\gamma(\x)$. On
the other hand if $\lambda h^*_\gamma(\x) \leq \btwo$ then
$L_\gamma(\lambda h^*_\gamma(\x), \b) = -\btwo \leq - \lambda
h^*_\gamma(\x)$. Thus
\begin{equation}
\label{eq:lhs}
\E(L_\gamma(\lambda h^*_\gamma(\x), \b) - L_\gamma( h^*_\gamma(\x), \b)) \Ind_{I_2}(\x)) 
 \leq (1 - \lambda) \E(h^*_\gamma(\x) \Ind_{I_2}(\x)) 
\end{equation}
This gives an upper bound for the left-hand side of inequality
\eqref{eq:i2i3}. We now seek to derive a lower bound on the
right-hand side. To do that, we analyze two different cases:
\begin{enumerate}
 \item $\lambda h^*_\gamma(\x)  \leq \bone$;
\item $\lambda h^*_\gamma(\x) > \bone$.
\end{enumerate}
In the first case, we know that $L_\gamma(h^*_\gamma(\x), \b) =
\frac{1}{\gamma} (h^*_\gamma(\x) - (1 + \gamma) \bone )> -\bone$
(since $h^*_\gamma(\x) > \bone$ for $(\x, \b) \in I_3$). Furthermore,
if $\lambda h^*_\gamma(\x) \leq \bone$, then, by definition
$L_\gamma(\lambda h^*_\gamma(\x), \b ) = \min(-\btwo, -\lambda
h^*_\gamma(\x) )\leq -\lambda h^*_\gamma(\x)$. Thus, we must have:
\begin{equation}
\label{eq:rhs1}
   L_\gamma(h^*_\gamma(\x), \b) - L_\gamma(\lambda h^*_\gamma(\x), \b) 
   > \lambda h^*_\gamma(\x) - \bone
   > (\lambda - 1) \bone 
   \geq (\lambda - 1) M,
\end{equation}
where we used the fact that $h^*_\gamma(\x) > \bone$ for the second
inequality and the last inequality holds since $\lambda - 1 < 0$.

We analyze the second case now. If $\lambda h^*_\gamma(\x) > \bone$,
then for $(\x, \b) \in I_3$ we have $L_\gamma(h^*_\gamma(\x), \b) -
L_\gamma(\lambda h^*_\gamma(\x), \b) = \frac{1}{\gamma} (1 - \lambda)
h^*_\gamma(\x)$. Thus, letting $\Delta(\x, \b) =
L_\gamma(h^*_\gamma(\x), \b) - L_\gamma(\lambda h^*_\gamma(\x), \b)$,
we can lower bound the right-hand side of \eqref{eq:i2i3} as:
\begin{align}
\E_{\x,\b}\Big[ \Delta(\x, \b)\Ind_{I_3}(\x)\Big] 
& =\E_{\x,\b}\Big[\Delta(\x, \b) \Ind_{I_3}(\x) \Ind_{\{\lambda
  h^*_\gamma(\x) > \bone\}}\Big] 
+ \E_{\x,\b}\Big[\Delta(\x, \b) \Ind_{I_3}(\x) \Ind_{\{\lambda h^*_\gamma(\x) \leq \bone\}}\Big]
  \nonumber \\
&\geq \frac{1 - \lambda}{\gamma}\E_{\x,\b} \Big[h^*_\gamma(\x)
\Ind_{I_3}(\x) \Ind_{\{\lambda h^*_\gamma(\x) > \bone\}}\Big] 
+ (\lambda - 1) M  \Pr\Big[h^*_\gamma(\x) > \bone \geq \lambda
h^*_\gamma(\x)\Big],
\label{eq:rhs2}
\end{align}
where we have used \eqref{eq:rhs1} to bound the second
summand. Combining inequalities \eqref{eq:i2i3}, \eqref{eq:lhs} and
\eqref{eq:rhs2} and dividing by $(1 - \lambda)$ we obtain the bound
\begin{equation*}
  \E_{\x, \b}\Big[h^*_\gamma(\x) \Ind_{I_2}(\x)\Big] 
   \geq \frac{1}{\gamma} \E_{\x, \b}
  \Big[ h^*_\gamma(\x) \Ind_{I_3}(\x)\Ind_{\{\lambda h^*_\gamma(\x) >
    \bone\}} \Big] 
 - M \Pr\Big[h^*_\gamma(\x) > \bone \geq \lambda h^*_\gamma(\x)\Big].
\end{equation*}
Finally, taking the limit $\lambda \rightarrow 1$, we obtain
\begin{equation*}
  \E_{\x, \b}\Big[h^*_\gamma(\x) \Ind_{I_2}(\x)\Big] \geq \frac{1}{\gamma}
  \E_{\x, \b} \Big[ h^*_\gamma(\x) \Ind_{I_3}(\x)\Big]. 
\end{equation*}
Taking the limit inside the expectation is justified by the bounded
convergence theorem and $\Pr[h^*_\gamma(\x) > \bone \geq \lambda
h^*_\gamma(\x)] \rightarrow 0$ holds by the continuity of probability
measures.
\end{proof}

\begin{proof}[Of Theorem~\ref{th:calibration}].
We can express the difference as
\begin{align}
\E_{\x,\b}\Big[L(h^*_\gamma(\x), \b)  - L_\gamma(h^*_\gamma(\x), \b)
\Big]  
&= \sum_{k=1}^4 \E_{\x, \b} \Big[(L(h^*_\gamma(\x), \b)- L_\gamma(h^*_\gamma(\x), \b)) \Ind_{I_k}(\x)\Big] \nonumber \\
& =  \E_{\x, \b} \Big[(L(h^*_\gamma(\x), \b) - L_\gamma(h^*_\gamma(\x), \b)) \Ind_{I_3}(\x) \Big]\nonumber\\
& = \E_{\x, \b} \Big[\frac{1}{\gamma}((1+\gamma) \bone -
h^*_\gamma(\x))\Ind_{I_3}(\x))\Big] .
\label{eq:lossdiff}
\end{align}
Furthermore, for $(\x,\b) \in I_3$, we know that $\bone <
h^*_\gamma(\x)$. Thus, we can bound \eqref{eq:lossdiff} by $\E_{\x,
  \b} [h^*_\gamma(\x) \Ind_{I_3}(\x)]$, which, by
Lemma~\ref{lemma:wstarprop}, is less than $\gamma \E_{\x,\b}
\Big[h^*_\gamma(\x) \Ind_{I_2}(\x)\big]$. We thus have:
\begin{equation*}
\E_{\x,\b} \Big[L(h^*_\gamma(\x), \b)\Big] -\E_{\x, \b}
  \Big[L_\gamma(h^*_\gamma(\x), \b) \Big] \\
\leq \gamma \E_{\x, \b}\Big[h^*_\gamma(\x) \Ind_{I_2}(\x)\Big] 
\leq \gamma \E_{\x, \b}\Big[\bone \Ind_{I_2}(\x)\Big] 
\leq  \gamma M,
\end{equation*}
since $h^*_\gamma(\x) \leq \bone$ for $(\x, \b) \in I_2$. 
\end{proof}
Notice that, since $L \geq L_\gamma$ for all $\gamma \geq 0$, it
follows easily from the proposition that $\cL_\gamma(h^*_\gamma)
\rightarrow \cL(h^*)$. Indeed, if $h^*$ is the best hypothesis in
class for the real loss, then the following inequalities are
straightforward:
\begin{align*}
0 \leq \cL_\gamma(h^*) - \cL_\gamma(h^*_\gamma) & \leq \cL(h^*) -
\cL_\gamma(h^*_\gamma) \\
& \leq \cL(h^*_\gamma) - \cL_\gamma(h^*_\gamma) \leq \gamma M
\end{align*}
The $1/\gamma$-Lipschitzness of $L_\gamma$ can be used to
prove the following generalization bound.
\begin{theorem}
\label{th:margin} 
Fix $\gamma \in (0, 1]$ and let $S$ denotes a sample of size $m$.
Then, for any $\delta > 0$, with probability at least $1 - \delta$ over
the choice of the sample $S$, for all $h \in H$, the following holds:
\begin{equation}
\label{eq:margin}
\cL_\gamma(h) \leq \h \cL_\gamma(h) + \frac{2}{\gamma} \R_m(H) + 
M \sqrt{\frac{\log \frac{1}{\delta}}{2m}}.
\end{equation}
\end{theorem}

\begin{proof}
  Let $\cL_{\gamma, H}$ denote the family of functions $\set{(\x, \b)
    \to L_\gamma(h(\x), b)\colon h \in H}$. The loss function
  $L_\gamma$ is $\frac{1}{\gamma}$-Lipschitz since the slope of the
  lines defining it is at most $\frac{1}{\gamma}$. Thus, using the
  contraction lemma (Lemma~\ref{lemma:contraction}) as in the proof of
  Proposition~\ref{prop:l1} gives $\R_m(\cL_{\gamma, H}) \leq
  \frac{1}{\gamma} \R_m(H)$. The application of a standard Rademacher
complexity bound to the family of functions $\cL_{\gamma, H}$ then
shows that for any $\delta > 0$, with probability at least $1 -
\delta$, for any $h \in H$, the following holds:
\begin{equation*}
\cL_\gamma(h) \leq \h \cL_\gamma(h) + \frac{2}{\gamma} \R_m(H) +
 M \sqrt{\frac{\log \frac{1}{\delta}}{2m}}.
\end{equation*}
\end{proof}

We conclude this section by presenting a stronger form of consistency
result. We will show that we can lower bound the generalization error
of the best hypothesis in class $\cL^* := \cL(h^*)$ in terms of that
of the empirical minimizer of $L_\gamma$, $\h h_\gamma:=
\argmin_{h \in H} \h \cL_\gamma(h)$.

\begin{theorem}
\label{th:lowerbound}
Let $M = \sup_{b \in \mathcal{B}} \bone$ and let $H$ be a hypothesis
set with pseudo-dimension $d = \text{Pdim}(H).$ Then for any $\delta >
0$ and a fixed value of $\gamma >0$, with probability at least $1 -
\delta$ over the choice of a sample $S$ of size $m$, the following
inequality holds:
\begin{equation*}
\cL(\h h_\gamma) \leq \cL^* + \frac{2 \gamma + 2}{\gamma} \R_m( H )
 + \gamma M  + 2M \sqrt{\frac{2d \log \frac{\e m}{d}}{m}} 
+ 2M \sqrt{\frac{\log\frac{2}{\delta}}{2m}}.
\end{equation*}
\end{theorem}
\begin{proof}
By Theorem~\ref{th:learningbound}, with probability at least
$1 - \delta/2$, the following holds:
\begin{equation}
  \label{eq:lbound}
  \cL(\h  h_\gamma) \leq \h \cL_S(\h h_\gamma) + 2 \R_m( H ) +
 2M \sqrt{\frac{2d \log \frac{\e m}{d}}{m}} + M \sqrt{\frac{\log\frac{2}{\delta}}{2m}}.
\end{equation}
Furthermore, applying Lemma~\ref{lemma:wstarprop} with the empirical
distribution induced by the sample, we can bound $\h \cL_S(\h
h_\gamma)$ by $\h \cL_\gamma(\h h_\gamma ) + \gamma M$. The first term
of the previous expression is less than $\h \cL_\gamma(h^*_\gamma)$ by
definition of $\h h_\gamma$. Finally, the same analysis as the one
used in the proof of Theorem~\ref{th:margin} shows that with
probability $1 - \delta/2$,
\begin{equation*}
  \h \cL_\gamma(h^*_\gamma) \leq \cL_\gamma(h_\gamma^*) + 
\frac{2}{\gamma} \R_m( H ) + M \sqrt{\frac{\log \frac{2}{\delta}}{2m}}.
\end{equation*}
Again, by definition of $h^*_\gamma$ and using the fact
that $L$ is an upper bound on $L_\gamma$, we can write
$\cL_\gamma(h^*_\gamma) \leq \cL_\gamma(h^*) \leq \cL(h^*)$. Thus,
\begin{align*}
\h \cL_S(\h h_\gamma) \leq \cL(h^*) + \frac{1}{\gamma} \R_m( H
) + 
M \sqrt{\frac{\log \frac{2}{\delta}}{2m}} + \gamma M.
\end{align*}
Combining this with \eqref{eq:lbound} and applying the union bound
yields the result.
\end{proof}
This bound can be extended to hold uniformly over all $\gamma$ at the
price of a term in $O\Big(\frac{\sqrt{\log \log_2
    \frac{1}{\gamma}}}{\sqrt{m}}\Big)$. Thus, for appropriate choices
of $\gamma$ and $m$ (for instance $\gamma \gg 1/m^{1/4}$) it would
guarantee the convergence of $\cL(\h h_\gamma)$ to $\cL^*$, a
stronger form of consistency.

These results are
reminiscent of the standard margin bounds with $\gamma$ playing the
role of a margin. The situation here is however somewhat
different. Our learning bounds suggest, for a fixed $\gamma \in (0,
1]$, to seek a hypothesis $h$ minimizing the empirical loss $\h
\cL_\gamma(h)$ while controlling a complexity term upper bounding
$\R_m(H)$, which in the case of a family of linear hypotheses could be
$\| h \|_K^2$ for some PSD kernel $K$.  Since the bound can hold
uniformly for all $\gamma$, we can use it to select $\gamma$ out of a
finite set of possible grid search values. Alternatively, $\gamma$ can
be set via cross-validation.

\section{Algorithms}
\label{sec:algorithms}

In this section we present algorithms for solving the optimization
problem for selecting the reserve price.  We start with the no-feature
case and then treat the general case.

\subsection{No feature case}
\label{sec:algo-no-feature}

We present a general algorithm to optimize sums of functions similar
to $L_\gamma$ or $L$ in the one-dimensional case.

\begin{figure}[t]
\centering
\begin{tabular}{cc}
\raisebox{-.3cm}{\includegraphics[scale=.5]{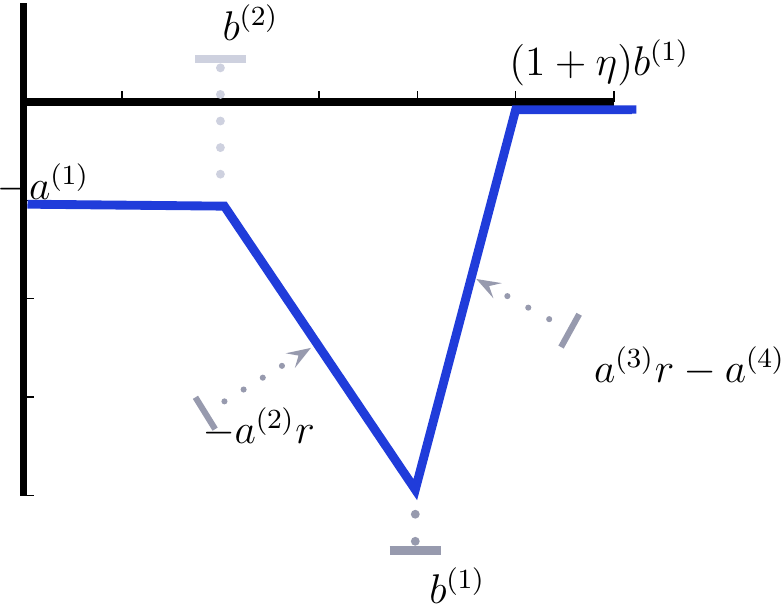}} &
\includegraphics[scale=.65]{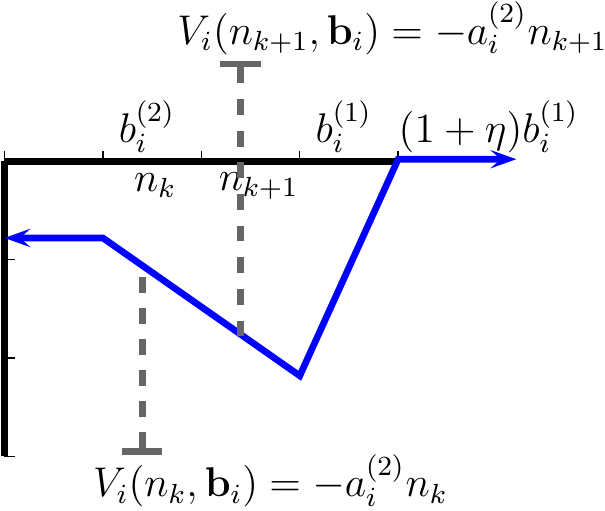}\\
(a) & (b)
\end{tabular}
\caption{(a) Prototypical $v$-function. (b) Illustration of the fact
that the definition of $V_i(r, \b_i)$ does not change on an interval
$[n_k, n_{k+1}]$.}
\vspace{-.1cm}
\label{fig:vloss}
\end{figure}

\begin{definition}
  We will say that function $V\colon \Rset \times \cB \to
  \Rset$ is a \emph{$v$-function} if it admits the following form:
\begin{equation*}
V(r, \b) = -a^{(1)} \Ind_{r \leq \btwo} -a^{(2)}r \Ind_{\btwo < r \leq \bone} +\\
(a^{(3)}r - a^{(4)})\Ind_{\bone < r < (1 + \eta) \bone},
\end{equation*}
with $a^{(1)} > 0$ and $\eta > 0$ constants and $a^{(1)}, a^{(2)}, a^{(3)}, 
a^{(4)}$ defined by $a^{(1)} = \eta a^{(3)} \btwo$, $a^{(2)} = \eta
a^{(3)}$, and $a^{(4)} = a^{(3)} (1 + \eta) \bone$.
\end{definition}
Figure~\ref{fig:vloss}(a) illustrates this family of loss functions.  A
$v$-function is a generalization of $L_\gamma$ and $L$. Indeed, any
$v$-function $V$ satisfies $V(r,\b) \leq 0$ and attains its minimum at
$\bone$. Finally, as can be seen straightforwardly from
Figure~\ref{fig:rho_loss}, $L_\gamma$ is a $v$-function for any
$\gamma \geq 0$.  We consider the following general problem of
minimizing a sum of $v$-functions:
\begin{equation}
\label{eq:1dproblem} 
\min_{r \geq 0} \ F(r) := \sum_{i = 1}^m V_i(r, \b_i).
\end{equation}
Observe that this is not a trivial problem since, for any fixed
$\b_i$, $V_i(\cdot, \b_i)$ is non-convex and that, in general, a sum of
$m$ such functions may admit many local minima. The following
proposition shows that the minimum is attained at one of the highest
bids, which matches the intuition.
\begin{proposition}
\label{prop:minproperty} 
  Problem \eqref{eq:1dproblem} admits a solution $r^*$ that satisfies
  $r^* = \bone_i$ for some $i \in [1, m]$.
\end{proposition}

In order to proof this proposition a pair of lemmas are required.

\begin{definition}
For any $r \in \Rset$, define the following subset of $\Rset$: 
\begin{align*}
  \Omega(r) = \{\epsilon | r < \bone_i \leftrightarrow r + \epsilon
  \leq \bone_i \ \forall i\} \\
\end{align*}
\end{definition}

We will drop the dependency on $r$ when this value is understood from context.
\begin{figure}[t]
\centering
\includegraphics[scale=.8]{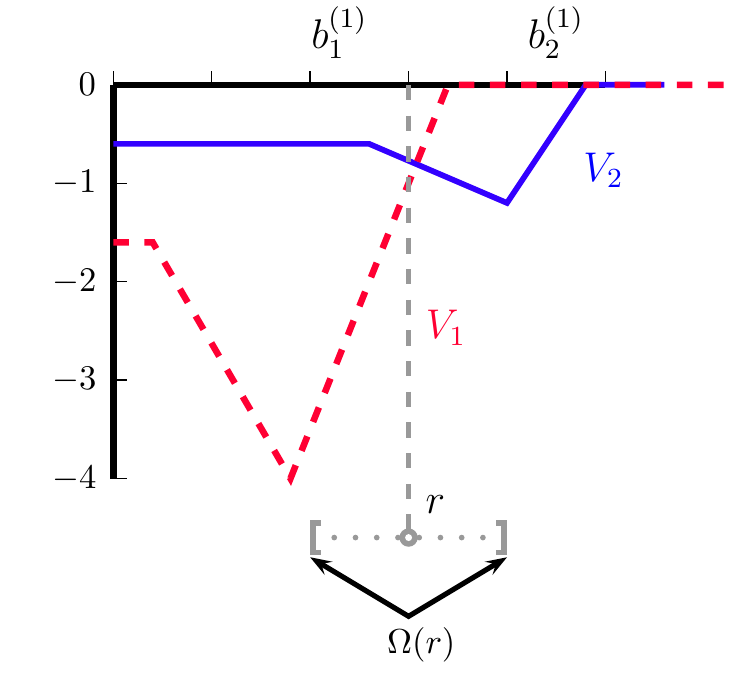}
\caption{Illustration of the region $\Omega(r)$. The functions $V_i$ are
  monotonic and concave when restricted to this region.}
\label{fig:omega}
\end{figure}

\begin{lemma}
\label{lemma:direction} Let $r \neq \bone_i$ for all $i$. If $\epsilon
> 0$ is such that $[-\e, \e] \subset \Omega(r)$ then $F(r + \epsilon)
< F(r)$ or $F(r - \epsilon) \leq F(r)$.
\end{lemma}

\begin{proof} 
  Let $v_i = V_i(r, \b_i)$ and $v_i(\e) = V_i(r + \epsilon, \b_i)$.
  For $\epsilon \in \Omega(r)$ define the sets $D(\epsilon) = \{i ~|~
  v_i(\e) \leq v_i\}$ and $I(\epsilon) = \{i ~|~ v_i(\e) > v_i\}$. If
\begin{equation*}
\sum_{i \in D(\e)} v_i + \sum_{i \in I(\e)} v_i > \sum_{i \in D(\e)} v_i(\e) + \sum_{i \in I(\e)} v_i(\e),
\end{equation*}
then, by definition, we have $F(r) > F(r + \e)$ and the result is
proven.  If this inequality is not satisfied, then, by grouping
indices in $D(\e)$ and $I(\e)$ we must have
\begin{equation}
\label{eq:didiff}
 \sum_{i \in D(\e)} v_i - v_i(\e) ~\leq~ \sum_{i \in I(\e)} v_i(\e) - v_i
\end{equation}
Notice that $v_i(\e) \leq v_i$ if and only if $v_i(-\e) \geq
v_i$. Indeed, the function $V_i(r + \eta, \b_i)$ is monotone for $\eta
\in [-\e, \e]$ as long as $[-\e, \e] \subset \Omega$ which is true by
the choice of $\e$.  This fact can easily be seen in
Figure~\ref{fig:omega}. Hence $D(\e) = I(-\e)$, similarly $I(\e) =
D(-\e)$ Furthermore, because $V_i(r + \eta, \b_i)$ is also concave for
$\eta \in [-\e, \e]$. We must have
\begin{equation}
  \label{eq:concavity}
  \frac{1}{2}(v_i(-\e) + v_i(\e)) \leq v_i.
\end{equation}
Using \eqref{eq:concavity}, the following inequalities are easily derived:
\begin{align}
  \label{eq:lhsbound}
  v_i(-\e) - v_i & \leq v_i - v_i(\e) & \qquad \text{for } i \in D(\e)\\
  \label{eq:rhsbound}   
  v_i(\e) - v_i & \leq v_i - v_i(-\e) & \qquad \text{for } i \in I(\e).
\end{align}
Combining inequalities \eqref{eq:lhsbound}, \eqref{eq:didiff} and \eqref{eq:rhsbound} we obtain
\begin{align*}
\sum_{i \in D(\e)} v_i(-\e) - v_i  &\leq \sum_{i \in I(\e)} v_i -
v_i(-\e)\\
\Rightarrow \quad \sum_{i \in I(-\e)} v_i(-\e) - v_i & \leq \sum_{i \in D(-\e)} v_i - v_i(-\e).
\end{align*}
By rearranging back the terms in the inequality we can easily see that
$F(r - \e) \leq F(r)$.
\end{proof}

\begin{lemma}
\label{lemma:step} 
Under the conditions of Lemma~\ref{lemma:direction}, if $F(r + \e)
\leq F(r)$ then $F(r + \lambda \e) \leq F(r)$ for every $\lambda$ that
satisfies $\lambda \e \in \Omega$ if and only if $\e \in \Omega$.
\end{lemma}

\begin{proof}
  The proof follows the same ideas as those used in the previous
  lemma. By assumption, we can write
\begin{equation}
  \label{eq:descent}
  \sum_{D(\e)} v_i - v_i(\e) \geq \sum_{i \in I(\e)} v_i(\e) - v_i.
\end{equation}
It is also clear that $I(\e) = I(\lambda \e)$ and $D(\e) = D(\lambda
\e)$. Furthermore, the same concavity argument of
Lemma~\ref{lemma:direction} also yields:
\begin{equation*}
  v_i(\e) \geq \frac{\lambda -1}{\lambda} v_i + \frac{1}{\lambda} v_i(\lambda \e),
\end{equation*}
which can be rewritten as
\begin{equation}
\label{eq:concavelambda}
\frac{1}{\lambda}(v_i - v_i(\lambda \e)) \geq v_i - v_i(\e).
\end{equation}
Applying inequality \eqref{eq:concavelambda} in \eqref{eq:descent} we obtain
\begin{equation*}
  \frac{1}{\lambda} \sum_{D(\lambda \e)} v_i - v_i(\lambda \e) \geq
  \frac{1}{\lambda} 
  \sum_{I(\lambda \e)} v_i(\lambda \e) - v_i.
\end{equation*}
Since $\lambda > 0$, we can multiply the inequality by $\lambda$ to
derive an inequality similar to \eqref{eq:descent} which implies that
$F(r + \lambda \e) \leq F(r)$.
\end{proof}

\begin{proof} (Of Proposition~\ref{prop:minproperty})
  Let $r \neq \bone_i$ for every $i$. By Lemma~\ref{lemma:direction},
  we can choose $\e\neq 0$ small enough with $F(r + \e) \leq
  F(r)$. Furthermore if $\lambda = \min_i \frac{|\bone_i - r|}{|\e|}$
  then $\lambda$ satisfies the hypotheses of
  Lemma~\ref{lemma:step}. Hence, $F(r) \geq F(r + \lambda \e) =
  F(b_{i^*})$, where $i^*$ is the minimizer of $\frac{|\bone_i -
    r|}{|\e|}$.
\end{proof}

Problem \eqref{eq:1dproblem} can thus be reduced to examining the value of the function
for the $m$ arguments $\bone_i$, $i \in [1, m]$. This yields a
straightforward method for solving the optimization which consists of
computing $F(\bone_i)$ for all $i$ and taking the minimum. But, since
the computation of each $F(\bone_i)$ takes $O(m)$, the overall
computational cost is in $O(m^2)$, which can be prohibitive for even
moderately large values of $m$.  

Instead, we present a combinatorial algorithm to solve the optimization
problem \eqref{eq:1dproblem} in $O(m \log m)$. Let $\mathcal{N} =
\bigcup_i \set{\bone_i, \btwo_i, (1+\eta) \bone_i}$ denote the set of all
\emph{boundary points} associated with the functions $V(\cdot,
\b_i)$. The algorithm proceeds as follows: first, sort the set
$\mathcal{N}$ to obtain the ordered sequence $(n_1, \ldots, n_{3m})$,
which can be achieved in $O(m \log m)$ using a comparison-based
sorting algorithm. Next, evaluate $F(n_1)$ and compute $F(n_{k+1})$
from $F(n_{k})$ for all $k$.

The main idea of the algorithm is the following: since the
definition of $V(\cdot, b_i)$ can only change at boundary points (see
also Figure~\ref{fig:vloss}(b)), computing $F(n_{k+1})$ from
$F(n_{k})$ can be achieved in constant time. Indeed, since between $n_k$ and
$n_{k+1}$ there are only two boundary points, we can compute
$V(n_{k+1} , \b_i)$ from $V(n_{k}, \b_i)$ by calculating $V$ for only
two values of $\b_i$, which can be done in constant time. We now give
a more detailed description and proof of correctness for the
algorithm.

\begin{proposition}
\label{prop:algorithm}
There exists an algorithm to solve optimization problem
\eqref{eq:1dproblem} in $O(m \log m)$.
\end{proposition}

\begin{proof}
  The pseudo-code for the desired algorithm is presented in Algorithm
  \ref{alg:sorting}.  Where $a_i^{(1)},..., a_i^{(4)}$ denote the
  parameters defining the functions $V_i(r, \b_i)$.
\begin{algorithm}
\caption{Sorting \label{alg:sorting}}
\begin{algorithmic}
\STATE{\strut $\mathcal{N} := \bigcup_{i=1}^m \{b_i^{(1)}, b_i^{(2)}, (1+\eta)
  b_i^{(1)}\}$;}
\STATE{(\strut $n_1,...,n_{3m})= {\bf Sort}(\mathcal{N})$;}
\STATE{\strut Set $\mathbf{c}_i:=(c_i^{(1)},c_i^{(2)},c_i^{(3)},c_i^{(4)}) = 0 $ for $i = 1, ..., 3m$;}
\STATE{\strut Set $c_1^{(1)} = -\sum_{i=1}^m a_i^{(1)} $;}
\FOR{$j = 2, ..., 3m$}
 \STATE{Set $\mat{c}_{j} = \mat{c}_{j-1}$;} 
 \IF{$n_{j-1} = \btwo_i$ for some $i$}
  \STATE{ $c_j^{(1)} = c_j^{(1)} + a_i^{(1)}$;} 
  \STATE{ $c_j^{(2)} = c_j^{(2)} - a_i^{(2)}$;}
 \ELSIF{$n_{j-1} = \bone_i$ for some $i$}
  \STATE{ $c_j^{(2)} = c_j^{(2)} + a_i^{(2)}$;}
  \STATE{ $c_j^{(3)} = c_j^{(3)} + a_i^{(3)}$;}
  \STATE{ $c_j^{(4)} = c_j^{(4)} - a_i^{(4)}$;}
 \ELSE
  \STATE{ $c_j^{(3)} = c_j^{(3)} - a_i^{(3)}$;}
  \STATE{ $c_j^{(4)} = c_j^{(4)} + a_i^{(4)}$;}
 \ENDIF
\ENDFOR
\end{algorithmic}
\end{algorithm}

We will prove that after running Algorithm~\ref{alg:sorting} we
can compute $F(n_j)$ in constant time using:
\begin{equation}
\label{eqn:evaluation}
F(n_j) = c_j^{(1)} + c_j^{(2)}  n_j + c_j^{(3)}  n_j + c_j^{(4)}. 
\end{equation}
This holds trivially for $n_1$ since by definition $n_1 \leq
\btwo_i$ for all $i$ and therefore $F(n_1) = -\sum_{i=1}^m
a_i^{(1)}$. Now, assume that \eqref{eqn:evaluation} holds for $j$, we
prove that it must also hold for $j + 1$. Suppose $n_j = b_i^2$
for some $i$ (the cases $n_j = \bone_i$ and $n_j = (1 + \eta) \bone_i$
can be handled in the same way).  Then $V_i(n_j, \b_i) = -a_i^{(1)}$
and we can write
\begin{equation*}
\sum_{k \neq i } V_k(n_j, \b_k) = F(n_j) - V(n_j, \b_i)
= (c^{(1)}_j + c_j^{(2)}n_j +  c_j^{(3)} n_j + c_j^{(4)}) + a_i^{(1)}.
\end{equation*}
Thus, by construction we would have:
\begin{align*}
c_{j+1}^{(1)} + c_{j+1}^{(2)}  n_{j+1} + c_{j+1}^{(3)}  n_{j+1} +
  c_{j+1}^{(4)}
& =  c^{(1)}_j + a_i^{(1)}+ (c_j^{(2)} - a_i^{(2)})n_{j+1} +
c_j^{(3)} n_{j+1} + c_j^{(4)}\\
& = (c^{(1)}_j + c_j^{(2)}n_{j+1} +  c_j^{(3)} n_{j+1} + c_j^{(4)}) +
a_i^{(1)} - a_i^{(2)}n_{j+1}\\
& = \sum_{k \neq i } V_k(n_{j+1}, \b_k) -a_i^{(2)} n_{j+1},
\end{align*}
where the last equality holds since the definition of $V_k(r, \b_k)$
does not change for $r \in [n_j, n_{j+1}]$ and $k\neq i$.  Finally,
since $n_j$ was a boundary point, the definition of $V_i(r, \b_i)$
must change from $-a_i^{(1)}$ to $-a_i^{(2)}r$, thus the last equation
is indeed equal to $F(n_{j+1})$. A similar argument can be given if
$n_j = \bone_i$ or $n_j = (1 + \eta) \bone_i$.

Let us analyze the complexity of the algorithm: sorting the set
$\mathcal{N}$ can be performed in $O(m \log m)$ and each iteration
takes only constant time. Thus the evaluation of all points can be
done in linear time. Having found all values, the
minimum can also be obtained in linear time.  Thus, the overall time
complexity of the algorithm is $O(m \log m)$.
\end{proof}

The algorithm just proposed can be straightforwardly extended
to solve the minimization of $F$ over a set of $r$-values bounded by
$\Lambda$, that is $\set{r\colon 0 \leq r \leq \Lambda}$. Indeed, we
then only need to compute $F(\bone_i)$ for $i \in [1, m]$ such that
$\bone_i < \Lambda$ and of course also $F(\Lambda)$, thus the
computational complexity in that regularized case remains $O(m \log
m)$.

\subsection{General case}

We first consider the case of a hypothesis set $H$ of linear functions
$\x \mapsto \w \cdot \x$ with bounded norm, $\| \w \| \leq \Lambda$,
for some $\Lambda \geq 0$. This can be immediately generalized to 
the case where a positive definite kernel is used.

The results of Theorem~\ref{th:margin} 
suggest seeking, for a fixed $\gamma \geq 0$, the vector $\w$ solution
of the following optimization problem: $\min_{\| \w \| \leq \Lambda}
\sum_{i = 1}^m L_\gamma(\w \cdot \x_i,\b_i)$.  Replacing the original
loss $L$ with $L_\gamma$ helped us remove the discontinuity of the
loss. But, we still face an optimization problem based on a sum of
non-convex functions. This problem can be formulated as a
DC-programming (difference of convex functions programming)
problem. Indeed, $L_\gamma$ can be decomposed as follows for all $(r,
\b) \in \cX \times \cB$: $L_\gamma(r, \b) = u(r, \b) - v(r, \b)$, with
the convex functions $u$ and $v$ defined by
\begin{align*}
 u(r,\b) &= -r \Ind_{r < \bone} + \tfrac{r - (1 + \gamma)
 \bone}{\gamma} \Ind_{r \geq \bone} \\
v(r, \b) &= (-r + \btwo) \Ind_{r < \btwo} + \tfrac{r - (1 +
  \gamma) \bone}{\gamma} \Ind_{r > (1 + \gamma) \bone}.
\end{align*}
Using the decomposition $L_\gamma = u - v$, our optimization problem
can be formulated as follows:
\begin{equation}
\label{opt:dc}
\min_{\w \in \Rset^N} \ U(\w) - V(\w) \qquad \text{subject to} \ \| \w \| \leq \Lambda,
\end{equation}
where $U(\w) = \sum_{i = 1}^m u(\w \cdot \x_i, \b_i)$ and $V(\w) =
\sum_{i = 1}^m v(\w \cdot \x_i, \b_i)$, which shows that it can be
formulated as a DC-programming problem.  The global minimum of the
optimization problem \eqref{opt:dc} can be found using a cutting plane
method \cite{HorstThoai1999}, but that method only converges in the
limit and does not admit known algorithmic convergence
guarantees.\footnote{Some claims of \cite{HorstThoai1999}, e.g.,
Proposition~4.4 used in support of the cutting plane algorithm, are
incorrect \cite{Tuy2002}.}  There exists also a branch-and-bound
algorithm with exponential convergence for DC-programming
\cite{HorstThoai1999} for finding the global minimum. Nevertheless, in
\cite{TaoAn1997}, it is pointed out that this type of combinatorial
algorithms fail to solve real-world DC-programs in high dimensions. In
fact, our implementation of this algorithm shows that the convergence
of the algorithm in practice is extremely slow for even moderately
high-dimensional problems. Another attractive solution for finding
the global solution of a DC-programming problem over a polyhedral
convex set is the combinatorial solution of Hoang Tuy
\cite{Tuy1964}. However, casting our problem as an instance of that
problem requires explicitly specifying the slope and
offsets for the piecewise linear function corresponding to a sum of
$L_\gamma$ losses, which admits an exponential cost in time and space.

An alternative consists of using the DC algorithm, a primal-dual
sub-differential method of Dinh Tao and Hoai An \cite{TaoAn1998},
(see also \cite{TaoAn1997} for a good survey).  This algorithm is
applicable when $u$ and $v$ are proper lower semi-continuous convex
functions as in our case. When $v$ is differentiable, the DC
algorithm coincides with the CCCP algorithm of Yuille and Rangarajan
\cite{YuilleRangarajan2003}, which has been used in several contexts
in machine learning and analyzed by
\cite{SriperumbudurLanckriet2009}. 

The general proof of convergence of the DC algorithm was given by
\cite{TaoAn1998}. In some special cases, the DC algorithm can be used
to find the global minimum of the problem as in the trust region
problem \cite{TaoAn1998}, but, in general, the DC algorithm or its
special case CCCP are only guaranteed to converge to a critical point
\cite{TaoAn1998,SriperumbudurLanckriet2009}.  Nevertheless, the number
of iterations of the DC algorithm is relatively small. Its
convergence has been shown to be in fact linear for DC-programming
problems such as ours \cite{YenPengWangLin2012}.  
The algorithm we are proposing goes one step further than that of
\cite{TaoAn1998}: we use DCA to find a local minimum but then restart
our algorithm with a new seed that is guaranteed to reduce the
objective function. Unfortunately, we are not in the same regime as in
the trust region problem of Dinh Tao and Hoai An \cite{TaoAn1998}
where the number of local minima is linear in the size of the
input. Indeed, here the number of local minima can be exponential in
the number of dimensions of the feature space and it is not clear to
us how the combinatorial structure of the problem could help us rule
out some local minima faster and make the optimization more tractable.

In the following, we describe more in detail the solution we propose
for solving the DC-programming problem \eqref{opt:dc}. The functions
$v$ and $V$ are not differentiable in our context but they admit a
sub-gradient at all points. We will denote by $\delta V(\w)$ an
arbitrary element of the sub-gradient $\partial V(\w)$, which
coincides with $\nabla V(\w)$ at points $\w$ where $V$ is
differentiable. The DC algorithm then coincides with CCCP, modulo the
replacement of the gradient of $V$ by $\delta V(\w)$. It consists of
starting with a weight vector $\w_0 \leq \Lambda$ and of iteratively
solving a sequence of convex optimization problems obtained by
replacing $V$ with its linear approximation giving $\w_t$ as a
function of $\w_{t - 1}$, for $t = 1, \ldots, T$:
$\w_t \in \argmin_{\| \w \| \leq \Lambda} \ U(\w) - \delta V(\w_{t - 1}) \cdot \w$.
This problem can be rewritten in our context as the following:
\begin{align}
\label{opt:qp}
\min_{\| \w \| \leq \Lambda, \mat{s}} & \ \sum_{i = 1}^m s_i - \delta V(\w_{t -
  1}) \cdot \w \mspace{10mu}
\\
\text{subject to} & \ (s_i \!\geq\! - \w \cdot \x_i) \!\wedge\!
\Big[\! s_i \!\geq\! \frac{1}{\gamma} \big(\w \cdot \x_i \!-\! (1 + \gamma)
\bone_i \big) \!\Big] \nonumber.
\end{align}
The problem is equivalent to a QP (quadratic-programming) problem
since the quadratic constraint can be replaced by a term of the form
$\lambda \| \w \|^2$ in the objective and thus can be tackled using
any standard QP solver.  We propose an algorithm that iterates along
different local minima, but with the guarantee of reducing the
function at every change of local minimum.  The algorithm is simple
and is based on the observation that the function $L_\gamma$ is positive
homogeneous. Indeed, for any $\eta > 0$ and $(r, \b)$,
\begin{align*}
L_\gamma (\eta r , \eta \b) 
&= -\eta \btwo \Ind_{\eta r <\eta \btwo} -\eta r \Ind_{\eta \btwo \leq \eta
r \leq \eta \bone} 
+ \frac{\eta r - (1 + \gamma) \eta \bone}{\gamma}
\Ind_{\eta \bone < \eta r < \eta (1 + \gamma) \bone} \\
& = \eta L_\gamma(r,\b).
\end{align*}

\begin{figure}[t]
\begin{algorithmic}
\hrule
\vspace{.1cm}
\STATE {\bf DC Algorithm}
\vspace{.1cm}
\hrule
\vspace{.1cm}
\STATE $\w \EQ \w_0 \mspace{90mu} \triangleright \text{initialization}$
\FOR {$t \geq 1$} 
\STATE $\w_t \EQ \text{\sc DCA}(\w) \quad \triangleright \text{{\sc DCA} algorithm}$
\STATE $\w \EQ \text{\sc Optimize}(\text{objective}, \text{fixed direction } \w_t/\|
  \w_t \|)$
\ENDFOR
\end{algorithmic}
\hrule
\caption{Pseudocode of our DC-programming algorithm.}
\vspace{-.3cm}
\label{fig:algo}
\end{figure}

Minimizing the objective function of \eqref{opt:dc} in a fixed
direction $\u$, $\| \u \| = 1$, can be reformulated as follows:
$\min_{0 \leq \eta \leq \Lambda} \ \sum_{i = 1}^m L_\gamma(\eta \u
\cdot \x_i, \b_i)$. Since for $\u \cdot \x_i \leq 0$ the function
$\eta \mapsto L_\gamma(\eta \u \cdot \x_i, \b_i)$ is constant equal to
$-\btwo_i$ the problem is equivalent to solving 
\begin{equation*}
\min_{0 \leq \eta
  \leq \Lambda} \sum_{\u \cdot \x_i >0} L_\gamma(\eta \u \cdot \x_i,
\b_i).
\end{equation*}
Furthermore, since $L_\gamma$ is positive homogeneous, for
all $i \in [1, m]$ with $\u \cdot \x_i > 0$, $L_\gamma(\eta \u \cdot
\x_i, \b_i) = (\u \cdot \x_i) L_\gamma(\eta, \b_i/(\u \cdot
\x_i))$. But $\eta \mapsto (\u \cdot \x_i) L_\gamma(\eta, \b_i/(\u
\cdot \x_i))$ is a $v$-function and thus the problem can efficiently
optimized using the combinatorial algorithm for the no-feature case
(Section~\ref{sec:algo-no-feature}). This leads to the optimization
algorithm described in Figure~\ref{fig:algo}.  The last step of each
iteration of our algorithm can be viewed as a \emph{line search} and
this is in fact the step that reduces the objective function the most
in practice. This is because we are then precisely minimizing the
objective function even though this is for some fixed direction. Since
in general this line search does not find a local minimum (we are
likely to decrease the objective value in other directions that are
not the one in which the line search was performed) running DCA helps
us find a better direction for the next iteration of the line search.

\section{Experiments}
\label{sec:experiments}

In this section we report the results of several experiments done on
synthetic, as well as realistic data demonstrating the benefits of our
algorithm. Since the use of features for reserve price optimization is
a completely novel idea, we are not aware of any baseline for comparison with
our algorithm. Therefore, its performance is measured against three
natural strategies that we now describe.

As mentioned before, a standard solution for solving this problem would
be the use of a convex surrogate loss. For that reason, we compare against
the solution of the regularized
minimization of the convex surrogate loss $L_\alpha$ shown in
Figure~\ref{fig:realLoss}(b) parametrized by $\alpha \in [0, 1]$ and
defined by
\begin{equation*}
L_\alpha(r, \b) \!=\!
\begin{cases} 
-r  \mspace{100mu} \text{if } r < \bone + \alpha (\btwo - \bone)\\
\left(\frac{(1 - \alpha) \bone + \alpha \btwo}{\alpha(\bone - \btwo)}
\right) \!(r - \bone) \mspace{40mu} \text{otherwise}.
\end{cases}
\end{equation*}
A second alternative consists of using ridge regression to 
estimate the first bid and use its prediction
as the reserve price.
A third algorithm consists of minimizing the loss while
ignoring the feature vectors $\x_i$, i.e., solving the problem
$\min_{r \leq \Lambda} \sum_{i =1}^n L(r, \b_i)$. It is worth
mentioning that this third approach is very similar to what
advertisement exchanges currently use to suggest reserve prices to
publishers. By using the empirical version of
Equation~\eqref{eq:revenue}, we see this algorithm is equivalent
to finding the empirical distribution of bids and optimizing the
expected revenue with respect to this empirical distribution as in
\citep{ostrovsky2011reserve} and
\citep{Cesa-BianchiGentileMansour2013}.

\subsection{Artificial data sets}
We generated $4$ different synthetic data sets with different
correlation levels between features and bids.
For all our experiments, the feature vectors $\x \in \Rset^{21}$ were
generated in the following way: $\tilde{\x} \in \Rset^{20}$ was
sampled from a standard Gaussian distribution and $\x =
(\tilde{\x},1)$ was created by adding an offset feature. We now describe the
bid generating process for each of the experiments as a function
of the feature vector $\x$. For our first three experiments, shown in 
Figure~\ref{fig:artificial}(a)-(c), the highest bid and second highest
bid were set to
$\max \Big( \Big|\sum_{i=1}^{21} x_i \Big| + \e_1,
\Big|\sum_{i=1}^{21} \frac{x_i}{2} + \e_2 \Big| \Big)_+$ and $\min \Big(
\Big|\sum_{i=1}^{21} x_i \Big| + \e_1, \Big|\sum_{i=1}^{21}
\frac{x_i}{2} + \e_2 \Big| \Big)_+$ respectively. Where $\e_i$ is a Gaussian
random variable with mean $0$. The standard deviation of the Gaussian
noise was varied over the set $\{0, 0.25, 0.5\}$.

For our last artificial experiment we used a generative model
supported on previous empirical observations
\citep{ostrovsky2011reserve, lahaiepennock2007}: bids were generated
by sampling two values from a lognormal
distribution with means $\x \cdot \w $ and $\frac{\x \cdot \w}{2}$ and
standard deviation $0.5$. Where $\w$ was random vector sampled from a
standard Gaussian distribution.

For all our experiments, the parameters $\Lambda, \gamma$ and $\alpha$
were tuned by using a validation set of the same number of examples as
the training set. The test set consisted of $5\mathord{,}000$ examples
drawn from the same distribution as the training set. Each experiment
was repeated 10 times and the mean revenue of each algorithm is shown in
Figure~\ref{fig:artificial}. The plots are normalized in such a
way that the revenue obtained by setting no reserve price is equal to
$0$ and the maximum possible revenue (which can be obtained by setting
the reserve price equal to the highest bid) is equal to $1$. The
performance of the ridge regression algorithm is not included in
Figure~\ref{fig:artificial}(d) as it was too inferior to be comparable with
the performance of the other algorithms.

By inspecting the results in Figure~\ref{fig:artificial}(a) we see
that even in the simplest, noiseless scenario our algorithm
outperforms all other techniques. It is also worth noticing that the
use of ridge regression is actually worse than using no features for
training. This fact is easily understood by noticing that the
square loss used in regression is symmetric. Therefore, we can expect
several reserve prices to be above the highest bid, which are
equivalent to zero revenue auctions. Another notable feature is that
as the noise level increases, the performance of feature-based
algorithms decreases. This is however true of any machine learning
algorithm: if the features become less relevant to the prediction
task, the performance of the algorithm will suffer. However, for the
convex surrogate algorithm something more critical occurs: the
performance of this algorithm actually decreases as the sample size
increases, which shows that in general learning with a convex
surrogate is not possible. This is an empirical verification of the
inconsistency result provided in Section~\ref{sec:surrogate}. This lack
of calibration can also seen in Figure~\ref{fig:artificial}(d), where
in fact the performance of this algorithm approaches the use of no
reserve price.

In order to better understand the reason behind the performance discrepancy
between feature-based algorithms, we analyze the reserve prices
offered by each algorithm. In Figure~\ref{fig:realdata}(a) we see that
the convex surrogate algorithm tends to offer lower reserve
prices. This should be intuitively clear as high reserve prices are
over-penalized by the chosen convex surrogate as shown in
Figure~\ref{fig:sumloss}.  On the other hand, reserve prices suggested
by the regression algorithm seem to be concentrated and symmetric
around their mean, therefore we can infer that about 50\% of the
reserve prices offered will be higher than the highest bid thereby
yielding zero revenue.

\begin{figure}[t]
  \centering
  \begin{tabular}{cc}
   (a) \includegraphics[scale=.5]{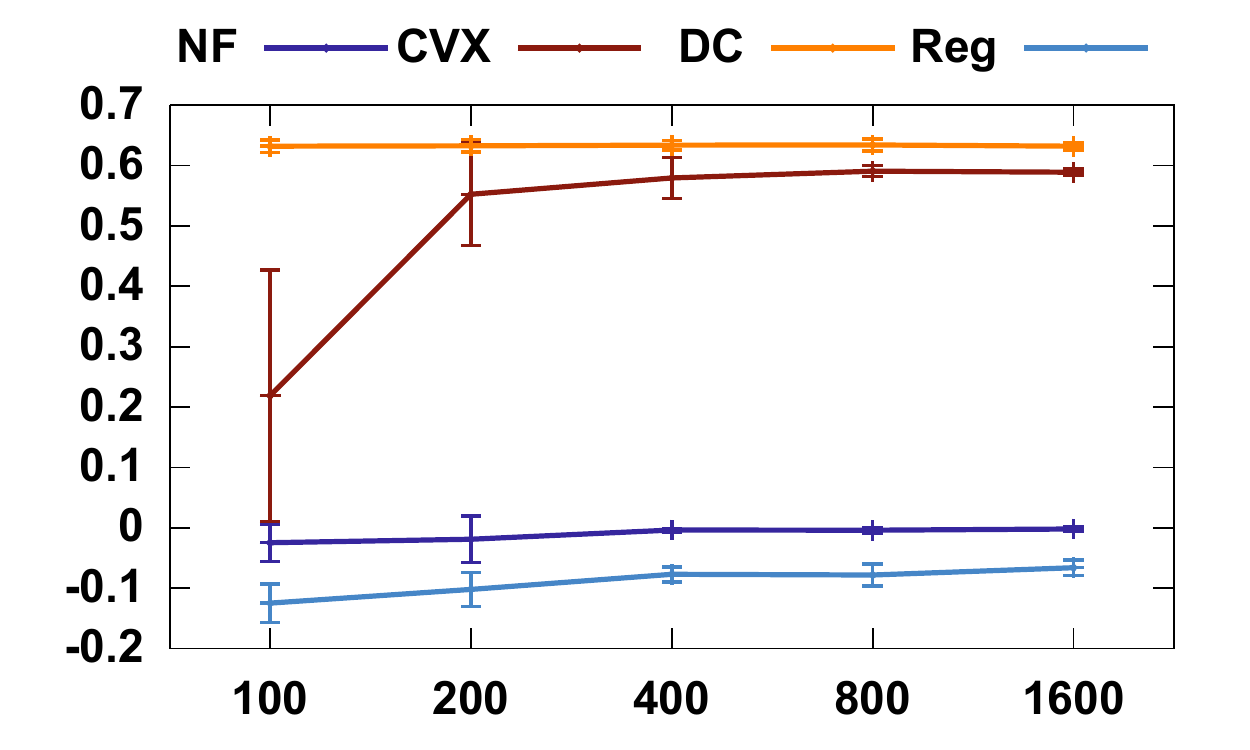}
    &(b)\includegraphics[scale=.5]{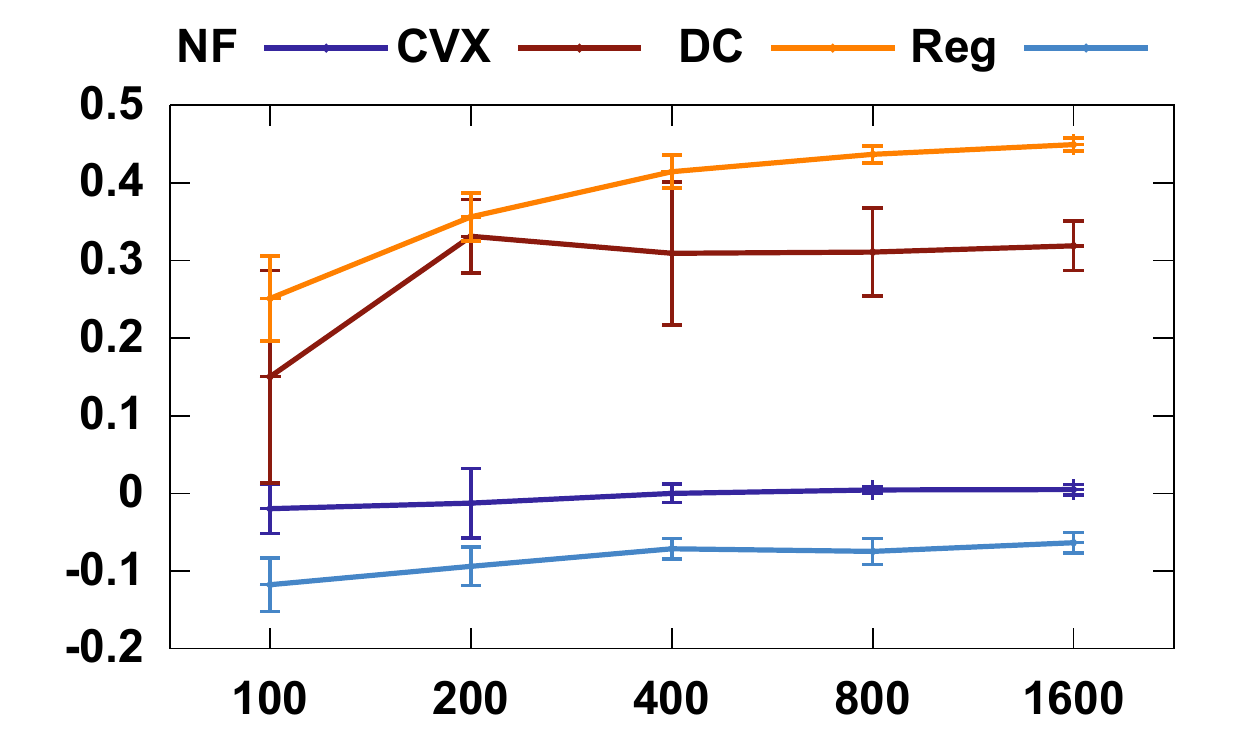} \\
    (c) \includegraphics[scale=.5]{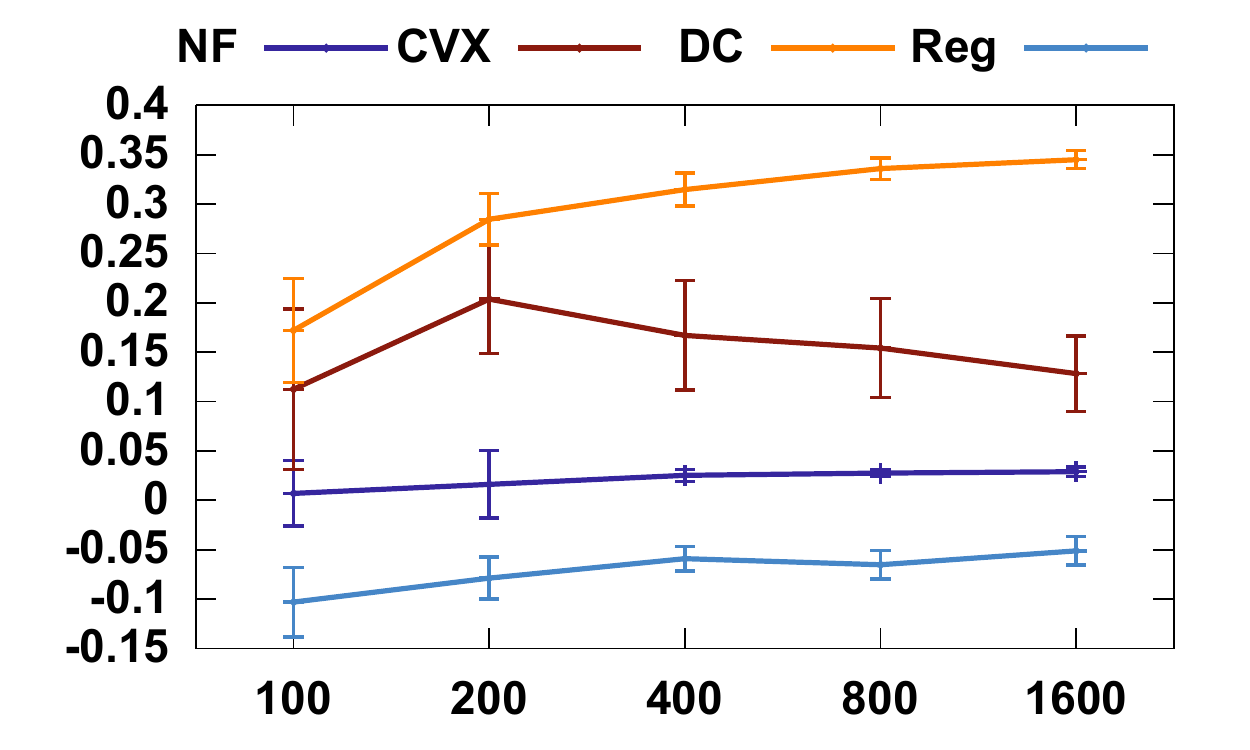} 
    &(d)\includegraphics[scale=.5]{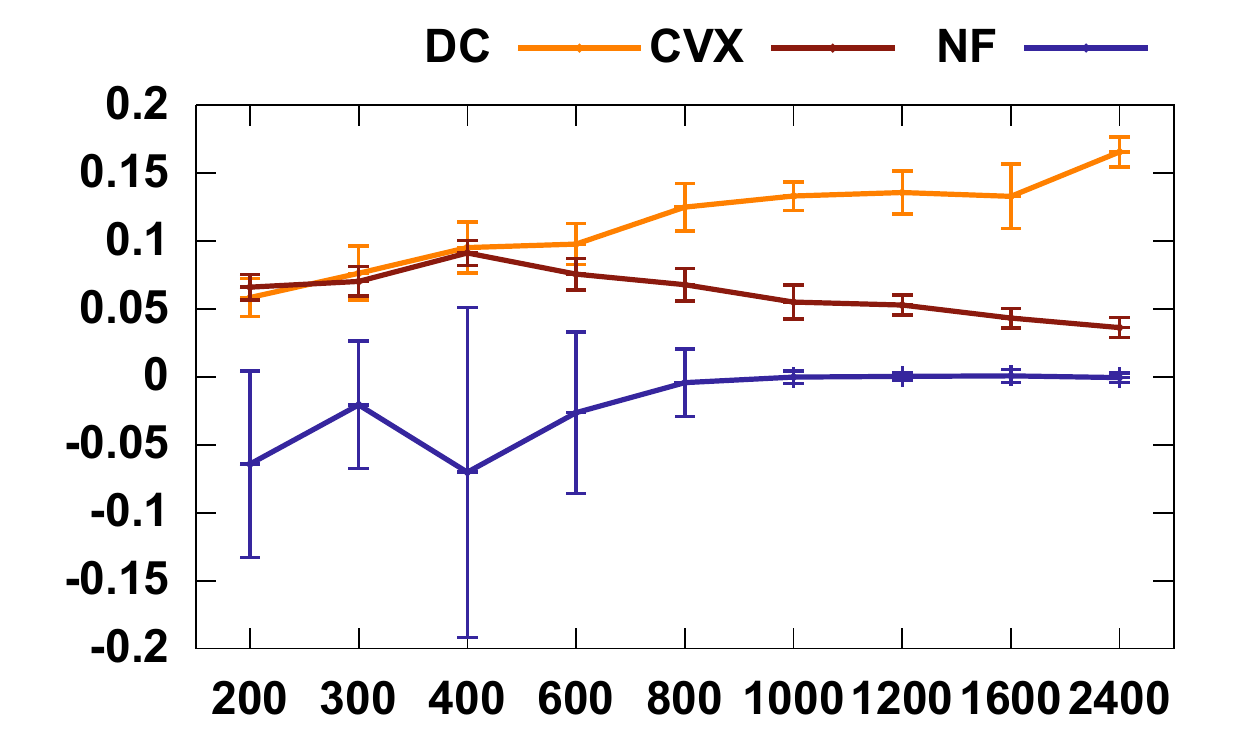}
  \end{tabular}
  \caption{Plots of expected revenue against sample size for different
algorithms: DC algorithm (DC), convex surrogate (CVX) and ridge
regression (Reg).  For (a)-(c) bids are generated with different noise
standard deviation (a) 0, (b) 0.25, (c) 0.5. The bids in (d) were
generated using a generative model.}
  \label{fig:artificial}
\end{figure}

\subsection{Realistic data sets}

Due to confidentiality and proprietary reasons, we cannot present
empirical results with AdExchange data. However, we were able to
secure an eBay data set consisting of collector sport cards. These
cards were sold using a second price auction with reserve and the full
data set can now be found in the following website: \newline \url{
http://cims.nyu.edu/~munoz/data}. Some other sources of auction data
sets are accessible
(e.g. \url{http://modelingonlineauctions.com/datasets}), but no
feature is available in those data sets.  To the best of our
knowledge, with exception of the data set used here, there is no
publicly available data set for online auctions including features
that could be readily used with our algorithm. The features include
information about the seller such as positive feedback percent, seller
rating and seller country; as well information about the card such as
whether the player is in the sport's hall of fame. The final dimension
of the feature vectors is $78$. The values of these features are both
continuous and categorical.

Since the highest bid is not reported by eBay, our algorithm cannot be
used straightforwardly on this data set. In order to generate highest
bids, we calculated the mean price of each object (each card was
generally sold more than once) and set the highest bid to be the
maximum between this average and the second highest bid.

On Figure~\ref{fig:realdata}(b) we show the revenue obtained by using
our DC algorithm, a convex surrogate and the algorithm that ignores
features. We also show the results obtained by using no reserve price
(NR) and highest possible revenue (HB). From the whole data set 2000
examples were randomly selected for training, validation and
testing. This experiment was repeated 10 times and
Figure~\ref{fig:realdata}(b) shows the mean revenue of each algorithm
and standard deviations.

\begin{figure}[t]
  \centering
\begin{tabular}{cc}
(a) \includegraphics[scale=.5]{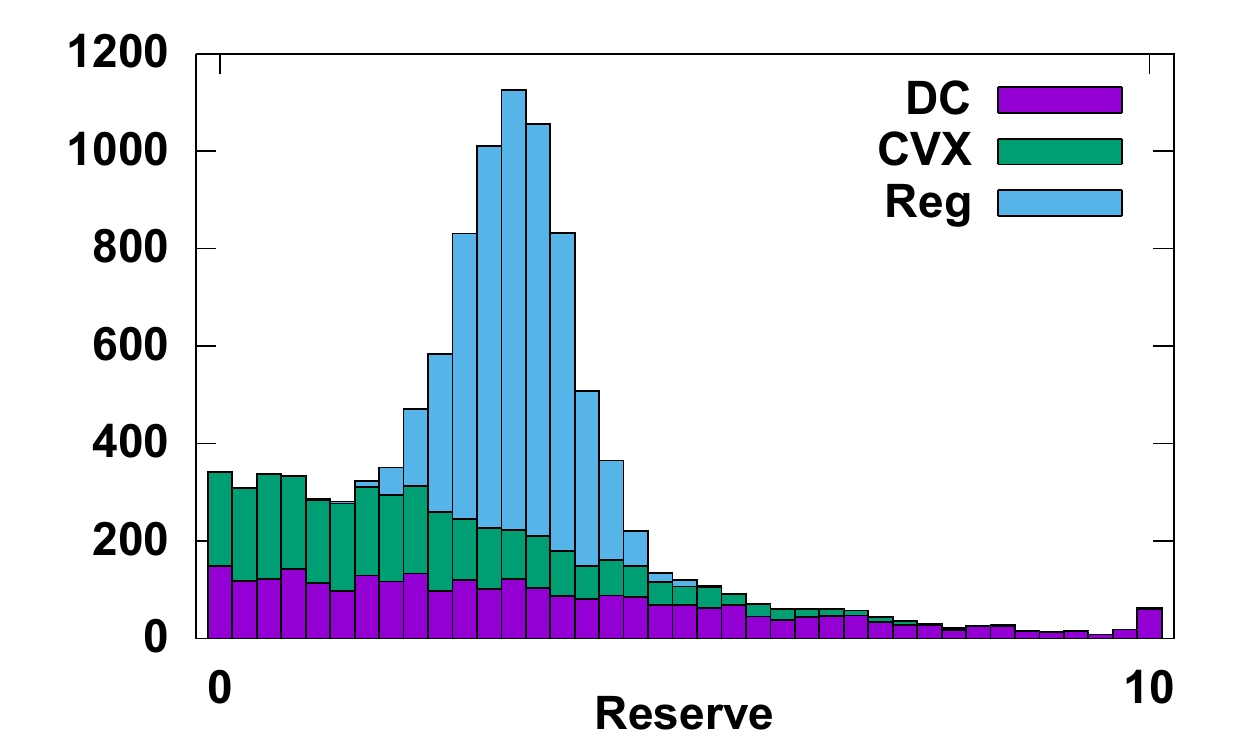} &
(b)\includegraphics[scale=.5]{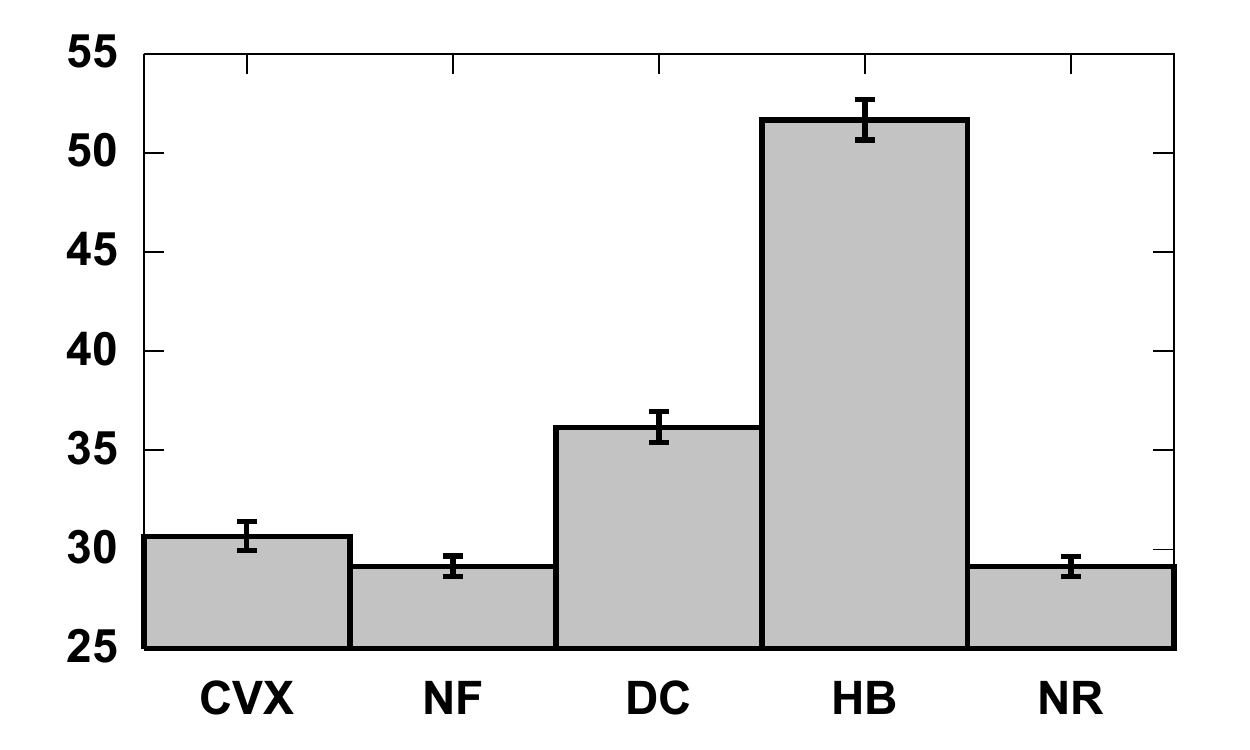}
\end{tabular}
  \caption{(a) Distribution of reserve prices for each algorithm. The
algorithms were trained on 800 samples using noisy bids with standard
deviation $0.5$. (b) Results of the eBay data set. }
  \label{fig:realdata}
\end{figure}
\section{Conclusion}

We presented a comprehensive theoretical and algorithmic analysis of
the learning problem of revenue optimization in second-price auctions
with reserve. The specific properties of the loss function for this
problem required a new analysis and new learning guarantees. The
algorithmic solutions we presented are practically applicable to
revenue optimization problems for this type of auctions in most
realistic settings. Our experimental results further demonstrate their
effectiveness. Much of the analysis and algorithms presented, in
particular our study of calibration questions, can also be of interest
in other learning problems.

\section*{Acknowledgments}

We thank Afshin Rostamizadeh and Umar Syed for several discussions
about the topic of this work and ICML reviewers for useful
comments. We also thank Jay Grossman for giving us access to the eBay
data set used in this paper. This work was partly funded by the NSF
award IIS-1117591.

\newpage
\bibliography{res}

\newpage
\appendix

\section{Contraction lemma}

The following is a version of Talagrand's contraction
lemma~\cite{LedouxTalagrand91}.  Since our definition of Rademacher
complexity does not use absolute values, we give an explicit proof
below.

\begin{lemma}
\label{lemma:contraction}
  Let $H$ be a hypothesis set of functions mapping $\cX$ to $\Rset$
  and $\Psi_1, \ldots, \Psi_m$, $\mu$-Lipschitz functions for some
  $\mu > 0$.  Then, for any sample $S$ of $m$ points $x_1, \ldots, x_m
  \in \cX$, the following inequality holds
\begin{align*}
\frac{1}{m} \E_{\ssigma} \left[\sup_{h \in H} \sum_{i = 1}^m \sigma_i
(\Psi_i \circ h) (x_i) \right] & \leq \frac{\mu}{m} \E_{\ssigma}
\left[\sup_{h \in H} \sum_{i = 1}^m \sigma_i h (x_i)\right] \\
&= \mu\, \h \R_S(H).
\end{align*}
\end{lemma}
\begin{proof}
  The proof is similar to the case where the functions $\Psi_i$ are
  all equal. Fix a sample $S = (x_1, \ldots, x_m)$. Then, we
  can rewrite the empirical Rademacher complexity as follows:
\begin{equation*}
\mspace{-20mu}\frac{1}{m}\E_{\ssigma}\Big[ \sup_{h \in H} \sum_{i = 1}^m \sigma_i (\Psi_i
\circ h)(x_i) \Big] = 
\mspace{10mu}\frac{1}{m} \E_{\sigma_1, \ldots, \sigma_{m - 1}} \Big[ \E_{\sigma_m}\Big[
\sup_{h \in H} u_{m - 1}(h)  + \sigma_m (\Psi_m \circ h)(x_m) \Big]
\Big] \,,
\end{equation*}
where $u_{m - 1}(h) = \sum_{i = 1}^{m - 1} \sigma_i (\Psi_i \circ
h)(x_i)$. Assume that the suprema can be attained and let $h_1, h_2\in
H$ be the hypotheses satisfying
\begin{flalign*}
& u_{m - 1}(h_1) + \Psi_m(h_1(x_m))  =   \sup_{h \in H} u_{m - 1}(h) +  \Psi_m (h(x_m)) \\
& u_{m - 1}(h_2) - \Psi_ m(h_2(x_m))  = \sup_{h \in H} u_{m - 1}(h)  -  \Psi_m(h(x_m)).
\end{flalign*}
When the suprema are not reached, a similar argument to what follows
can be given by considering instead hypotheses that are $\e$-close to the
suprema for any $\e > 0$.

By definition of expectation, since $\sigma_m$ uniform distributed
over $\set{-1, + 1}$, we can write
\begin{align*}
\E_{\sigma_m}\Big[ \sup_{h \in H} u_{m - 1}(h)  + \sigma_m (\Psi_m
\circ h)(x_m) \Big] &
= \frac{1}{2} \sup_{h \in H} u_{m - 1}(h)  +  (\Psi_m
\circ h)(x_m) 
+ \frac{1}{2} \sup_{h \in H} u_{m -  1}(h)  -  (\Psi_m \circ h)(x_m) \\
& = \frac{1}{2} [u_{m - 1}(h_1) + (\Psi_m \circ h_1)(x_m)]
+ \frac{1}{2} [u_{m - 1}(h_2) - (\Psi_m \circ h_2)(x_m)].
\end{align*}
Let $s = \sgn( h_1(x_m) - h_2(x_m) )$.  Then, the previous equality implies
\begin{align*}
\E_{\sigma_m}\Big[ \sup_{h \in H} u_{m - 1}(h)  + \sigma_m (\Psi_m
\circ h)(x_m) \Big]
& =  \frac{1}{2} [u_{m - 1}(h_1) + u_{m - 1}(h_2) + s \mu ( h_1(x_m) -
h_2(x_m))] \\
& = \frac{1}{2} [u_{m - 1}(h_1) + s \mu h_1(x_m)] 
 + \frac{1}{2} [u_{m -1}(h_2) - s \mu h_2(x_m)] \\
& \leq \frac{1}{2} \sup_{h \in H} [u_{m - 1}(h) + s \mu h(x_m)] 
 +\frac{1}{2} \sup_{h \in H} [u_{m -1}(h) - s \mu h(x_m)] \\
& = \E_{\sigma_m}\Big[ \sup_{h \in H} u_{m - 1}(h) + \sigma_m \mu
h(x_m) \Big],
\end{align*}
where we used the $\mu-$Lipschitzness of $\Psi_m$ in the first
equality and the definition of expectation over $\sigma_m$ for the
last equality.  Proceeding in the same way for all other $\sigma_i$'s
($i \neq m$) proves the lemma.
\end{proof}

\end{document}